\documentclass{article}

\usepackage{dilab_arxiv}

\usepackage{amsmath,amsfonts,bm}
\usepackage{amssymb,amsthm}

\def\eqref#1{equation~\ref{#1}}

\def\1{\bm{1}}

\def\vu{{\bm{u}}}

\def\mJ{{\bm{J}}}

\def\mR{{\bm{R}}}

\DeclareMathAlphabet{\mathsfit}{\encodingdefault}{\sfdefault}{m}{sl}
\SetMathAlphabet{\mathsfit}{bold}{\encodingdefault}{\sfdefault}{bx}{n}

\usepackage[capitalize,noabbrev]{cleveref}

\usepackage{hyperref}
\usepackage{url}
\newtheorem{theorem}{Theorem}[section]
\newtheorem{lemma}[theorem]{Lemma}

\newtheorem{corollary}[theorem]{Corollary}
\newtheorem{definition}[theorem]{Definition}
\newtheorem{assumption}[theorem]{Assumption}

\usepackage{amsmath,amsfonts,amssymb}
\usepackage{mathtools}
\usepackage{algorithm}
\usepackage{multirow}
\usepackage{algpseudocode}
\usepackage{enumerate}
\usepackage{array}

\title{Sharp Convergence and Sampling Trade-offs for Riemannian Diffusion under Nonnegative Ricci Curvature}
\runningtitle{Riemannian Diffusion under Nonnegative Ricci Curvature}
\date{arXiv preprint, \today}

\paperlogo{\includegraphics[height=1.5cm]{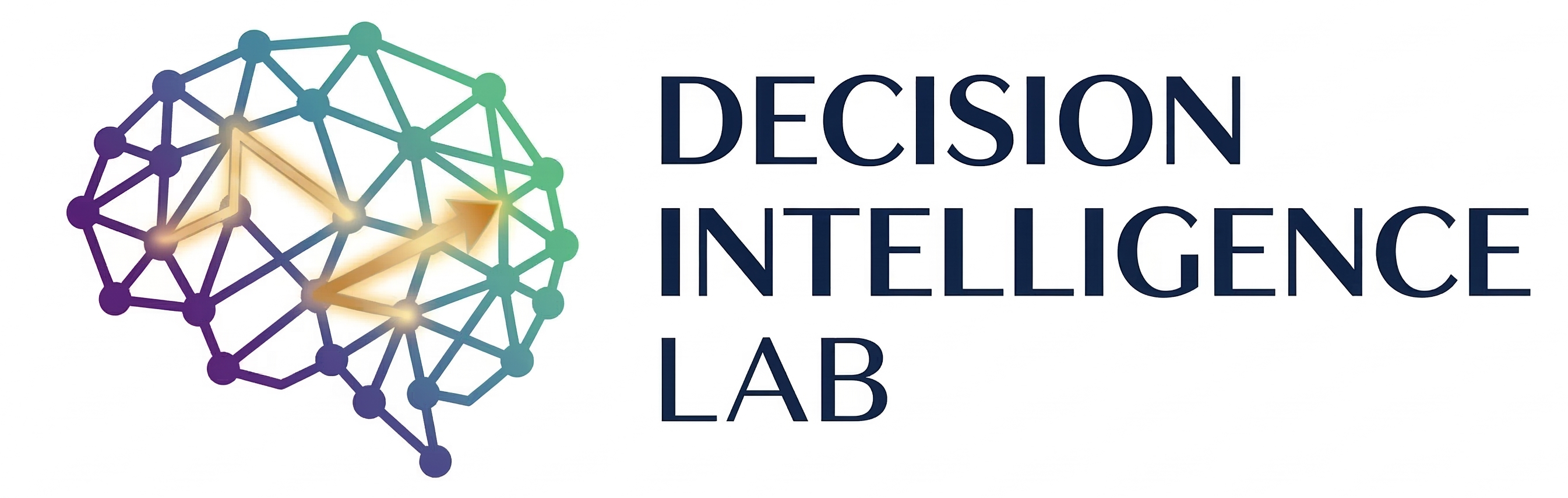}}

\author{
  Yuhao Liu$^{1}$ and Longbo Huang$^{1\,\text{\faEnvelope}}$
  \\[0.3em]\normalfont
  $^1$Institute for Interdisciplinary Information Sciences, Tsinghua University
  \\
  \quad \text{\faEnvelope}\ Correspondence: longbohuang@tsinghua.edu.cn
}

\begin{document}

\maketitle
\thispagestyle{fancy}

\begin{abstract}
Diffusion models have emerged as state-of-the-art generative models, with recent extensions from Euclidean spaces to Riemannian manifolds. However, existing convergence guarantees for Riemannian diffusion models typically require $\tilde{O}(\mathrm{poly}(d,T)/\epsilon^2)$ score evaluations, with potentially unfavorable dependence on the dimension. In this work, we develop a general framework that separates score discretization from Brownian-motion simulation and allows multiple geodesic random-walk steps per score evaluation. Under nonnegative Ricci curvature assumption and an exact Brownian-motion simulation oracle, we show that $\tilde{O}(d/\epsilon^2)$ score evaluations suffice to achieve an $\epsilon^2$ KL divergence from the target distribution, matching the existing convergence rate of Euclidean diffusion models. We further show that $\tilde{O}(d^4T/\epsilon^2)$ geodesic random-walk steps suffice to approximate the required drifted Brownian motion to $\epsilon$ total variation error. Combining these results yields a sampling scheme with $\tilde{O}(d/\epsilon^2)$ score evaluations and $\tilde{O}(d^4T/\epsilon^2)$ geodesic random-walk steps, motivating multiple random-walk steps between consecutive score evaluations. Our results provide a sharper characterization of the convergence and sampling complexity of Riemannian diffusion models.
\end{abstract}

\section{Introduction}

Diffusion models \citep{sohl2015deep} have emerged as some of the most successful generative models across a wide range of applications, including image and video generation. By modeling the generation process as stochastic differential equations (SDEs) \citep{song2021scorebased}, diffusion model can recover data from pure noise. In particular, the gradient of the log-density, commonly referred to as the \emph{score function}, naturally arises in the generation-time dynamics.

Most existing applications of diffusion models focus on Euclidean data. However, in many scientific and engineering applications, such as protein modeling and robotics, the underlying data have non-Euclidean structure that can be naturally represented by Riemannian manifolds. This has motivated several works to extend diffusion models from Euclidean spaces to Riemannian manifolds \citep{debortoli2022riemannian,huang2022riemannian,lou2023scaling}. In this setting, the generation are modeled by SDEs on manifolds, with the diffusion component given by Riemannian Brownian motion.

For diffusion models in both Euclidean spaces and Riemannian manifolds, practical sampling requires discretizing the reverse-time process, and the algorithms typically require many discretization steps to achieve a small sampling error. The convergence theory of sampling algorithms for Euclidean diffusion models has been extensively studied \citep{chen2023sampling,chen2023improved,benton2024dlinear,li2025odt}. In contrast, the convergence theory for Riemannian diffusion models remains comparatively underdeveloped. Existing convergence bounds suffer from unfavorable dependence on the dimension $d$ \citep{xu2026polynomial}, motivating a sharper analysis of the score discretization error. 

Moreover, many scientific datasets are supported on manifolds with additional geometric structure, among which \emph{nonnegative Ricci curvature} turns to be particularly important. Examples include climate and earth events modeled on the sphere $\mathbb{S}^2$ \citep{debortoli2022riemannian,huang2022riemannian}, orientations and rotations in robotics modeled on spheres $\mathbb{S}^d$ and special orthonormal groups $\mathrm{SO}(d)$ \citep{max2024riemannian}, and torsion angles in RNA and protein sequences modeled on flat tori $\mathbb{T}^d=(\mathbb{S}^1)^d$ \citep{lou2023scaling,cheng2025consistency}. Although these manifolds have been widely used in generation experiments, existing theoretical analyses typically impose only relatively general geometric assumptions, leaving a substantial gap between the theoretical guarantees and the settings encountered in practice.

More importantly, unlike in Euclidean space, Riemannian Brownian motion generally cannot be simulated exactly, introducing an additional source of approximation error. Existing algorithms \citep{debortoli2022riemannian,xu2026polynomial} typically perform one Brownian-motion simulation step for each score evaluation. However, as we show in our analysis, the error arising from Brownian-motion simulation may scale differently from the score discretization error. This suggests that restricting the algorithm to a single Brownian-motion simulation step per score evaluation may not provide the most efficient allocation of computational resources. In particular, when score evaluations are substantially more expensive than Brownian-motion simulation steps, it may be preferable to perform multiple Brownian-motion simulation steps between consecutive score evaluations.

In this paper, we develop a general framework for analyzing Riemannian diffusion samplers that separates the errors arising from score discretization and Brownian-motion simulation. This framework allows us to analyze multiple Brownian-motion simulation steps per score evaluation and thereby characterize the trade-off between these two sources of error. Under nonnegative Ricci curvature assumption, we obtain a sharper bound on the number of score evaluations required for accurate sampling and provide an explicit analysis of approximating Riemannian Brownian motion using geodesic random walks. Our results reveal that the number of expensive score evaluations can be substantially reduced by allocating additional, comparatively inexpensive computation to Brownian-motion simulation, providing theoretical guidance for accelerating sampling.

\vspace{-0.1cm}

\paragraph{Our Contributions.}

We summarize our contributions as follows.

\vspace{-0.1cm}

\begin{itemize}

\item We develop a general sampling framework (Section \ref{sec:framework}) that is sufficiently flexible to accommodate a broader class of sampling algorithms. The framework explicitly decomposes the sampling error into two components: i) the error arising from discretizing the reverse SDE, and ii) the error arising from approximating Brownian motion.

\item Under our general framework, we show that, when exact simulation of Riemannian Brownian motion is available, $O\left(\epsilon^{-2}d\log^2(T/\delta)\right)$ score evaluations suffice to achieve an $\epsilon^2$ KL divergence from the target distribution on Riemannian manifolds with nonnegative Ricci curvature (Theorem \ref{thm:sde}). This result substantially improves the existing convergence guarantees for Riemannian diffusion models. Moreover, it matches, up to logarithmic factors, the standard $\tilde O(d/\epsilon^2)$ complexity when degenerating to the Euclidean case. In addition, we also improves the existing convergence rate without Ricci curvature assumption.

\item We further analyze the approximation of Riemannian Brownian motion using geodesic random walks. We show that $O\left(\epsilon^{-2}d^4T\log(T/\delta)\right)$ geodesic random-walk steps suffice to achieve an $\epsilon$ total variation error (Theorem \ref{thm:bm}).

\item Finally, we combine the complexities for reverse-SDE discretization and Brownian-motion simulation to characterize the trade-off between score evaluations and Brownian-motion simulation steps. In particular, our analysis yields a theoretically preferable strategy that performs multiple Brownian-motion simulation steps per score evaluation (Corollary \ref{cor:num-of-steps}).

\end{itemize}

\renewcommand{\thempfootnote}{\fnsymbol{mpfootnote}}
\begin{table}[htbp]
\begin{minipage}{\textwidth}
    \centering
    \setlength{\extrarowheight}{1pt}
    \renewcommand{\arraystretch}{1.2}
    \begin{tabular}{>{\raggedright\arraybackslash}m{2.0cm}|c|c|c|c}
         \hline
         & Space & Target Error & \# Score Evaluations & \# GRW Steps\\
         \hline
         \citet{chen2023improved} & Euclidean & $\mathrm{KL}\le\epsilon^2$ & $\tilde{O}(\epsilon^{-2}d^2)$ & -\\
         \hline
         \citet{benton2024dlinear} & Euclidean & $\mathrm{KL}\le\epsilon^2$ & $\tilde{O}(\epsilon^{-2}d)$ & - \\
         \hline
         \citet{li2025odt} & Euclidean & $\mathrm{TV}\le\epsilon$ & $\tilde{O}(\epsilon^{-1}d)$ & - \\
         \hline
         \citet{xu2026polynomial} & Manifold & $\mathrm{TV}\le\epsilon$ & $\tilde{O}\left(\frac{\mathrm{poly}(\delta^{-1},T,d)}{\epsilon^2}\right)$ \footnote{This complexity bound is at least $\tilde{O}(\epsilon^{-2}d^6T^2\delta^{-3})$, which comes from their Lemma 2. They did not show their explicit polynomial factor in the Brownian motion simulation error.} & $\tilde{O}\left(\frac{\mathrm{poly}(\delta^{-1},T,d)}{\epsilon^2}\right)$ \\
         \hline
         \textbf{Our Results (Ric $\ge0$)} & \multirow{2}{*}{\vspace{-0.45cm}Manifold} & \multirow[c]{2}{*}{\vspace{-0.45cm}$\mathrm{TV}\le\epsilon$} & $\tilde{O}\left(\epsilon^{-2}d\right)$ & \multirow{2}{*}{\vspace{-0.45cm}$\tilde{O}\left(\epsilon^{-2}d^4T\right)$} \\
         \cline{1-1}\cline{4-4}
         \textbf{Our Results (General)} &  &  & $\tilde{O}\left(\epsilon^{-2}d^2T^2\right)$ & \\
         \hline
    \end{tabular}
    \caption{Comparison with previous analyses for diffusion models.}
    \label{tab:comparison}
\end{minipage}
\end{table}


\vspace{-0.2cm}

\section{Related Works}

\paragraph{Convergence Analysis for Euclidean Diffusion Models.}

\citet{chen2023sampling} developed a standard Girsanov-based framework for analyzing the sampling complexity of SDE-based generative models, including DDPMs \citep{ho2020ddpm}. Their analysis was subsequently refined by \citet{chen2023improved,benton2024dlinear}, who established sharper KL convergence bounds under relatively mild assumptions on the score functions and data distributions. In particular, \citet{benton2024dlinear} showed that $\tilde{O}(\epsilon^{-2}d)$ score evaluations suffice to achieve an $\epsilon^2$-KL error. For convergence in total variation (TV) distance, early works obtained guarantees by combining KL bounds with Pinsker's inequality. More recently, \citet{li2025odt} developed a sharper analysis that reduces the complexity of achieving an $\epsilon$-TV error to $\tilde{O}(\epsilon^{-1}d)$. In parallel, several works have focused on improving the sampling algorithms themselves to further reduce the computational complexity \citep{li2024fasternonasymptoticconvergencediffusionbased,li2024acceleratingconvergencescorebaseddiffusion,liang2024broadeningtargetdistributionsaccelerated,jain2026sharp}.

\paragraph{Riemannian Diffusion Models.}

\citet{debortoli2022riemannian} extended diffusion models from Euclidean spaces to Riemannian manifolds by formulating the forward and reverse processes as stochastic differential equations on manifolds. Independently, \citet{huang2022riemannian} developed a related framework from a maximum-likelihood perspective. \citet{lou2023scaling} further improved the computation and sampling of heat kernels on Riemannian symmetric spaces, enabling Riemannian diffusion models to be applied to high-dimensional real-world tasks.

Several works have studied the design and analysis of sampling methods on Riemannian manifolds \citep{cheng2022efficient,guan2025accuracy}. Regarding the convergence theory of Riemannian diffusion models, \citet{debortoli2022riemannian} established Wasserstein error bounds for their sampling method. More recently, \citet{xu2026polynomial} established the first polynomial convergence rate for Riemannian diffusion models. However, their complexity bound has a high-degree polynomial dependence on the dimension $d$, leaving substantial room for improvement. In this work, we provide a sharper convergence analysis and show that, under nonnegative Ricci curvature, the number of score evaluations required for sampling can be reduced to $O\left(\epsilon^{-2}d\log^2(T/\delta)\right)$, which matches the standard $\tilde{O}(d/\epsilon^2)$ score-evaluation complexity of Euclidean diffusion models up to logarithmic factors.
\section{Problem Formulation: Riemannian Diffusion Model}

In this section, we briefly introduce the Riemannian diffusion model proposed by \citet{debortoli2022riemannian}, and then formally propose the general sampling framework we are going to analyze. 

\subsection{Diffusion Model in Euclidean Spaces.}

We first briefly recall diffusion models in Euclidean spaces \citep{song2021scorebased,ho2020ddpm}. Let $p_0$ be a data distribution on $\mathbb{R}^d$. We consider the Ornstein-Uhlenbeck noising process 
$$\dif X_t=-X_t\dif t+\sqrt{2}\dif B_t,\quad X_0\sim p_0,$$
where $B_t$ is a standard $d$-dimensional Brownian motion. Denote the marginal density of $X_t$ by $p_t$, then $p_t$ converges to the standard Gaussian distribution $\mathcal{N}(0,I_d)$. Fix a terminal time $T>0$, the time-reversed process $Y_t=X_{T-t}$ satisfies
\begin{equation}
    \dif Y_t=(Y_t+2\nabla\log p_{T-t}(Y_t))\dif t+\sqrt{2}\dif B_t.
    \label{eq:reverse-euclidean}
\end{equation}
Thus, sampling from $p_0$ can be performed by simulating SDE (\ref{eq:reverse-euclidean}), initialized from $\mathcal{N}(0,I_d)$.

In practice, the score $\nabla\log p_{T-t}(\cdot)$ is approximated by a neural network $\hat{s}_t(\cdot)$. Given a time discretization $0=t_0<t_1<\cdots<t_{N-1}<t_N=T-\delta$,where $\delta>0$ is an early-stopping time (see, e.g., \citealt{song2020improved}), Euler-Maruyama discretization of (\ref{eq:reverse-euclidean}) leads to 
$$\hat{Y}_{t_{k+1}}=\hat{Y}_{t_k}+(t_{k+1}-t_k)\left[\hat{Y}_{t_k}+2\hat{s}_{t_k}(\hat{Y}_{t_k})\right]+\sqrt{2(t_{k+1}-t_k)}w_k,\quad w_k\sim\mathcal{N}(0,I_d).$$
Equivalently, the discretized process has a frozen drift on each interval $[t_k,t_{k+1})$:
\begin{equation}
    \dif\hat{Y}_t=(\hat{Y}_{t_k}+2\hat{s}_{t_k}(\hat{Y}_{t_k}))\dif t+\sqrt{2}\dif B_t,\quad t\in[t_k,t_{k+1}),\quad k=0,\dots,N-1.
    \label{eq:reverse-euclidean-discretized}
\end{equation}

The convergence of such sampling schemes has been extensively studied. In particular, \citet{benton2024dlinear} showed that, under minimal assumptions, $\tilde{O}(d/\epsilon^2)$ score evaluations suffice to ensure that the KL divergence between $p_\delta$ and the distribution of $\hat{Y}_{T-\delta}$ is at most $\epsilon^2$. This nearly linear dependence on the dimension will serve as the benchmark for our analysis of Riemannian diffusion models.

\subsection{Diffusion Models on Riemannian Manifolds}

We now introduce the Riemannian generalization of diffusion models \citep{debortoli2022riemannian}. Throughout this paper, we restrict our attention to compact manifolds. Let $\mathcal{M}$ be a connected, $d$-dimensional compact Riemannian manifold with metric $g$. We denote by $\rho(x,y)$ the Riemannian distance between $x,y\in\mathcal{M}$, by $\nabla$ the Riemannian gradient, and by $\Delta$ the Laplace-Beltrami operator. We use $\langle\cdot,\cdot\rangle=\langle\cdot,\cdot\rangle_g$ to denote the Riemannian inner product.

Let $p_0$ be the data distribution on $\mathcal{M}$. To define a diffusion process on $\mathcal{M}$, we replace the Euclidean Brownian motion by Riemannian Brownian motion. We use the following characterization, which is convenient for our analysis.

\begin{definition}[Brownian motion on $\mathcal{M}$]
    Let $(X_t)_{t\ge0}$ be a continuous stochastic process on $\mathcal{M}$, and let $b_t$ be a time-dependent vector field on $\mathcal{M}$. We say that $X_t$ is a Riemannian Brownian motion on $\mathcal{M}$ with drift $b_t$, if, for any smooth function $f:[0,\infty)\times\mathcal{M}\to\mathbb{R}$, the process $M_t^f$ defined by
    $$M_t^f:=f(t,X_t)-f(0,X_0)-\int_0^t\left(\partial_s+\langle b_s(X_s),\nabla\rangle+\frac{1}{2}\Delta\right)f(s,X_s)\dif s$$
    is a continuous local martingale. Equivalently, we say $\dif X_t=b_t(X_t)\dif t+\dif B_t.$
    \label{definition:bm}
\end{definition}

For compact manifolds, we consider the following forward noising process:
\begin{equation}
    \dif X_t=\dif B_t,\quad X_0\sim p_0.
    \label{eq:forward-riemannian}
\end{equation}
For $t>0$, $X_t$ admits a smooth density $p_t(\cdot)$ with respect to the Riemannian volume measure, satisfying the heat equation
$$\partial_tp_t(x)=\frac{1}{2}\Delta p_t(x).$$
Since $\mathcal{M}$ is compact and connected, the operator $-\Delta$ has nonnegative eigenvalues. So, $p_t$ converges to the uniform distribution $u_{\mathcal{M}}$, given by the normalized Riemannian volume measure.

For a terminal time $T>0$, define the reverse process by $Y_t=X_{T-t}$. Under standard regularity conditions, the reverse process satisfies \citep{debortoli2022riemannian}
\begin{equation}
    \dif Y_t=\nabla\log p_{T-t}(Y_t)\dif t+\dif B_t.
    \label{eq:reverse-riemannian}
\end{equation}
Thus, given access to the Riemannian score $\nabla\log p_{T-t}(\cdot)$ and an initial sample approximately distributed according to $u_{\mathcal{M}}$, sampling from $p_0$ reduces to simulating the reverse SDE.

In practice, the score is approximated by a neural network $\hat{s}_t(\cdot)\approx\nabla\log p_{T-t}(\cdot)$.  A common approach to approximate the reverse SDE is based on geodesic random walks, which can be viewed as the Riemannian analogue of Euler-Maruyama. Let $0=t_0<t_1<\cdots<t_N=T-\delta$ be the time discretization, and denote the step sizes by $h_k:=t_{k+1}-t_k$. Starting from $\hat{Y}_0$, the geodesic random walk considered in \citet{debortoli2022riemannian} is given by
\begin{equation}
    \hat{Y}_{t_{k+1}}=\exp_{\hat{Y}_{t_k}}\left(h_k\hat{s}_{t_k}(\hat{Y}_{t_k})+\sqrt{h_k}w_k\right),\quad k=0,\dots,N-1,
    \label{eq:original-grw-algorithm}
\end{equation}
where $\exp_x\mathcal{M}\to\mathcal{M}$ is the exponential map and $w_k$ is a standard Gaussian vector in $T_{\hat{Y}_{t_k}}\mathcal{M}$ with respect to any orthonormal basis. Equivalently, each step first performs an Euler-Maruyama update in the tangent space and then maps the resulting tangent vector to $\mathcal{M}$ through the exponential map.

Unlike in Euclidean space, however, Riemannian Brownian motion generally does not admit an exact finite-step simulation by a Gaussian increment. Therefore, the error of (\ref{eq:original-grw-algorithm}) has two distinct sources: the error from discretizing the reverse SDE and the error from approximating Riemannian Brownian motion. The latter arises from the local approximation $\sqrt{h_k}w_k$ on the tangent space, which does not exactly reproduce a Riemannian Brownian increment over a finite time interval. Existing convergence analyses typically treat these two sources of error together, resulting in bounds with unfavorable dependence on the dimension $d$. For example, \citet{xu2026polynomial} showed that, under suitable assumptions on the manifold and score functions, we require $O\left(\mathrm{poly}(d,T,\delta^{-1})/\epsilon^2\right)$ steps to ensure that the total variation (TV) distance between the distributions of $X_\delta$ and $\hat{Y}_{T-\delta}$ is at most $\epsilon$.

\subsection{A Sampling Framework: Brownian Motion Simulation Oracle}

\label{sec:framework}

As discussed in the previous section, sampling from the reverse SDE on a Riemannian manifold involves two distinct difficulties: discretizing the reverse SDE, and simulating Riemannian Brownian motion. To separate these two sources of difficulty, we introduce a general sampling framework based on a Brownian-motion simulation oracle (i.e., a black-box).

Consider a time discretization $0=t_0<t_1<\cdots<t_{N-1}<t_N=T-\delta$, and let $h_k:=t_{k+1}-t_k$ denote the length of the $k$-th time interval. We assume access to a general Brownian-motion simulation oracle $\mathcal{O}$. Given a time-independent vector field $b$ on $\mathcal{M}$, an initial point $x\in\mathcal{M}$, and a time horizon $h>0$, the oracle $\mathcal{O}(b,x,h)$ returns a random point whose distribution approximates the endpoint at time $h$ of the Riemannian SDE
\begin{equation}
    \dif X_t=b(X_t)\dif t+\dif B_t,\quad X_0=x,\quad t\in[0,h].
    \label{eq:sde-oracle}
\end{equation}
The oracle is allowed to be implemented using an arbitrary simulation procedure; in particular, it need not be restricted to a single geodesic random-walk step.

The high-level idea of our framework is simple. At each step $k$, we first evaluate the score estimator at the current state to obtain an instantaneous drift $s_k$. We then extend this drift to a vector field $\tilde{s}_k$ on $\mathcal{M}$, and use the Brownian-motion simulation oracle to evolve the process over the interval $[t_k,t_{k+1}]$. Formally, given the initial state $\hat{Y}_{0}$, our sampling framework is summarized in Algorithm \ref{alg:framework}.
\begin{algorithm}[htbp]
\caption{Sampling Framework with Simulation Oracle.}
\label{alg:framework}
\begin{algorithmic}[1]

\Require $\{t_k\}_{k=0}^N$, $\mathcal{O}(\cdot,\cdot,\cdot)$, $\hat{s}_{t_k}(\cdot)$, $\hat{Y}_{0}$.

\For{$k=0,1,\dots,N-1$}
    \State $s_k\gets\hat{s}_{t_k}(\hat{Y}_{t_k})$, $h_k\gets t_{k+1}-t_k$.
    \State Define a vector field $\tilde{s}_k(\cdot)$ using $\hat{Y}_{t_k}$ and $s_k$, satisfying $\tilde{s}(\hat{Y}_{t_k})=s_k$.
    \State $\hat{Y}_{t_{k+1}}\gets\mathcal{O}(\tilde{s}_k,\hat{Y}_{t_k},h_k)$.
\EndFor

\end{algorithmic}
\end{algorithm}

Suppose first that the oracle provides an exact simulation of SDE (\ref{eq:sde-oracle}). Then the resulting sampling process can be represented by the SDE
\begin{equation}
    \dif \hat{Y}_t=\tilde{s}_k(\hat{Y}_t)\dif t+\dif B_t,\quad t\in[t_k,t_{k+1}).
    \label{eq:reverse-riemannian-estimate}
\end{equation}
Thus, when the Brownian-motion simulation is exact, the only approximation relative to the true reverse process (\ref{eq:reverse-riemannian}) comes from replacing the time-dependent score $\nabla\log p_t(\cdot)$ by the frozen vector field $\tilde{s}_k(\cdot)$ on each time interval. Consequently, the distributional error between $\hat{Y}_{T-\delta}$ and the target $X_\delta=Y_{T-\delta}$ can be analyzed entirely through a comparison between the two SDEs (\ref{eq:reverse-riemannian}) and (\ref{eq:reverse-riemannian-estimate}).

This formulation is closely analogous to Euler-Maruyama discretization in Euclidean space. The key difference is that, in $\mathbb{R}^d$, an instantaneous drift $s_k$ is naturally identified with a constant vector field, so one can simply set $\tilde{s}_k(x)=s_k$. On a Riemannian manifold, however, $s_k\in T_{\hat{Y}_{t_k}}\mathcal{M}$ belongs to a tangent space that depends on the current point, and therefore cannot be directly evaluated at other points on $\mathcal{M}$. Constructing a suitable extension $\tilde{s}_k$ is therefore an additional geometric component of the discretization. We provide an appropriate construction and analyze its resulting error in Section \ref{sec:sde}.

For a general, potentially inexact simulation oracle, the total distributional error can be decomposed into two components. The first is the error incurred by the frozen-drift approximation assuming that the Brownian-motion simulation is exact; the second is the additional error introduced by the simulation oracle itself. This decomposition allows us to analyze score discretization and Brownian-motion simulation separately. A similar decomposition has also been used implicitly in the analysis of \citet{xu2026polynomial}, but our framework makes the two sources of error explicit and allows the Brownian-motion simulation procedure to be designed independently of the score discretization.

As a special case, suppose that the oracle is defined by $\mathcal{O}(b,x,h)=\exp_{x}\left(h\cdot b(x)+\sqrt{h}w\right)$ where $w\in T_x\mathcal{M}$ is an independent Gaussian vector. Then Algorithm \ref{alg:framework} reduces exactly to the geodesic random walk introduced in the previous section, regardless of how the vector field $\tilde{s}_k(\cdot)$ is defined away from $\hat{Y}_{t_k}$. More generally, however, our framework permits arbitrarily sophisticated Brownian-motion simulation procedures while requiring only a single score evaluation per interval. This distinction is important because score evaluations, typically implemented by evaluating a neural network, can be substantially more expensive than Brownian-motion simulation steps. In particular, in Section \ref{sec:bm}, we instantiate the oracle $\mathcal{O}$ by performing $n_{\mathrm{bm}}$ smaller geodesic random-walk steps within each interval $[t_k,t_{k+1}]$. Thus, we can increase the accuracy of Brownian-motion simulation without increasing the number of score evaluations. This provides the basic mechanism underlying our computational trade-off between score evaluations and Brownian-motion simulation steps.

In the following sections, we first analyze the idealized setting in which the Brownian-motion simulation oracle is exact. Under nonnegative Ricci curvature, we show in Section \ref{sec:sde} that nearly linear dependence on the dimension $d$ is sufficient for the number of score evaluations. We then analyze the additional error introduced by approximating the oracle using $n_{\mathrm{bm}}$ geodesic random-walk steps in Section \ref{sec:bm}.
\section{Sharp Convergence Rate with Exact Simulation Oracle under Nonnegative Ricci Curvature}

\label{sec:sde}

In this section, we analyze Algorithm \ref{alg:framework} on manifolds with nonnegative Ricci curvature under the assumption that the simulation oracle $\mathcal{O}$ exactly solves the Riemannian SDE (\ref{eq:sde-oracle}). Our goal is therefore to compare the true reverse process (\ref{eq:reverse-riemannian}) with the estimated process (\ref{eq:reverse-riemannian-estimate}). We adopt the construction of \citet{xu2026polynomial}. Let $\eta:\mathbb{R}\to\mathbb{R}$ be a $C^2$ cut-off function, satisfying
$$\eta(r)=\begin{cases}
    1,&|r|\le\frac{1}{2},\\
    0,&|r|\ge 1,
\end{cases}\quad|\eta'(r)|\le C,\quad|\eta''(r)|\le C,$$
where $C>0$ is a universal constant. For a scale parameter $\omega>0$, we define $\eta_{\omega}:r\mapsto \eta(r/\omega)$. Thus, $\eta_\omega$ is supported on $[-\omega,\omega]$. Fix a discretization time $t_k$, let $y:=\hat{Y}_{t_k}=y$, $s_k:=\hat{s}_{t_k}(y)$. With these information, we then define the vector field in Algorithm \ref{alg:framework} by
\begin{equation}
    \tilde{s}_k(x):=\eta_{\omega}(\rho(x,y))\cdot\dif[\exp_{y}\left(\log_yx\right)](s_k),\quad x\in\mathcal{M}.
    \label{eq:extended-score}
\end{equation}
where $\log_y x$ denotes the inverse exponential map on the normal neighborhood of $y$. Here, we canonically identify $T_{\log_yx}(T_y\mathcal{M})$ with $T_y\mathcal{M}$, since $T_y\mathcal{M}$ is a vector space. The construction has a simple geometric interpretation. In normal coordinates centered at $y$, the vector $s_k$ is kept constant, and is then mapped back to the manifold through the differential of the exponential map. The cut-off function localizes this construction to the geodesic ball $B_y(\omega)$, while also ensuring that $\tilde{s}_k$ is globally well-defined and smooth. In particular, $\tilde{s}_k(y)=s_k$, and within the smaller ball $B_y(\omega/2)$, the coefficients of $\tilde{s}_k$ in the normal coordinate representation $x\mapsto\log_yx$ are constant.

We take the scale parameter as $\omega:=c/\sqrt{d}$ in the construction (\ref{eq:extended-score}) for a sufficiently small universal constant $c>0$. Under the geometric assumptions imposed later, this choice can ensure that the relevant normal neighborhood is contained within the injectivity radius, and the Riemannian metric $g$ is sufficiently close to the identity matrix in the local coordinate representation. For manifolds with an explicit exponential map, such as the sphere, the vector field (\ref{eq:extended-score}) can be computed explicitly.

Now, we introduce the assumptions required for our theory.

\begin{assumption}
    The compact manifold $\mathcal{M}$ satisfies the following geometric conditions.
    \begin{itemize}
        \item \textbf{Bounded diameter.} $\mathrm{diam}(\mathcal{M})=\sup\{\rho(x,y):x,y\in\mathcal{M}\}\le D$.
        \item \textbf{Bounded injectivity radius.} The injective radius $i(x)$ satisfies $i(x)\ge r_{i}>0$.
        \item \textbf{Bounded curvature.} The Riemannian curvature tensor $R$ satisfies $\max\{\Vert R\Vert,\Vert\nabla R\Vert,$ $\Vert\nabla^2 R\Vert\}\le C_{\mathrm{Rm}}$, where the norms are operator norms.
        \item \textbf{Nonnegative Ricci curvature.} For every $x\in\mathcal{M}$ and $X\in T_x\mathcal{M}$, $\mathrm{Ric}(X,X)\ge0$.
    \end{itemize}
    \label{assumption:manifold}
\end{assumption}

The first three conditions in Assumption \ref{assumption:manifold} are the same geometric assumptions considered by \citet{xu2026polynomial}. For any fixed compact smooth Riemannian manifold, these quantities are finite and the injectivity radius is strictly positive. In our analysis, we additionally require the corresponding constants $D$, $r_i$, and $C_{\mathrm{Rm}}$ to be independent of the dimension $d$, since we are interested in the dependence of the convergence rate on $d$. This is satisfied, for example, on the unit $d$-sphere $\mathbb{S}^d$, where we can take $D=\pi,r_i=\pi$ and $C_{\mathrm{Rm}}=1$, which are all independent of $d$.

The \textbf{nonnegative Ricci curvature} assumption is the main additional geometric condition in our analysis. It does not hold for general compact Riemannian manifolds; nevertheless, it is satisfied by many commonly used manifolds, including the unit sphere $\mathbb{S}^d$ and the special orthogonal group $\mathrm{SO}(n)$ equipped with its standard bi-invariant metric. We discuss how the analysis can be extended to manifolds with general Ricci curvature in Appendix \ref{appendix:ricci}. As we show later, nonnegative Ricci curvature plays a crucial role in obtaining the nearly $d$-linear KL convergence bound.

\begin{assumption}[Score Error]
    $\sum_{k=0}^{N-1}h_k\mathbb{E}_{x\sim p_{T-t_k}}\left[\Vert\nabla\log p_{T-t_k}(x)-\hat{s}_{t_k}(x)\Vert^2\right]\le\epsilon_{\mathrm{score}}^2$.
    \label{assumption:score-estimate}
\end{assumption}

Assumption \ref{assumption:score-estimate} is standard in the theoretical analysis of diffusion models \citep{chen2023sampling} and measures the cumulative mean-squared error of the score estimator along the forward diffusion process. In particular, our analysis only requires an error bound in expectation and does not impose a uniform pointwise bound on the score estimation error.

With Assumptions \ref{assumption:manifold} and \ref{assumption:score-estimate} in place, we are ready to state our main convergence result. Recall that $(p_{t})$ denotes the marginal density of the forward process. Let $(\hat{q}_t)$ denote the marginal density of estimated reverse process $(\hat{Y}_{t})_{t\ge0}$ defined by SDE (\ref{eq:reverse-riemannian-estimate}).

\begin{theorem}
    Fix $T>0$ and $\delta>0$. Let the step sizes be given by $h_k=\min\{h(T-t_k),T-\delta-t_k\}$, where $0<h<\frac{1}{2}$ is a scale parameter. Suppose that the vector field used in Algorithm \ref{alg:framework} is defined according to (\ref{eq:extended-score}). Then, under Assumptions \ref{assumption:manifold} and \ref{assumption:score-estimate},
    $$\begin{aligned}
        \mathrm{TV}(p_{\delta},\hat{q}_{T-\delta})^2&\lesssim\mathrm{KL}(p_{\delta}\Vert\hat{q}_{T-\delta})\\
        &\lesssim\underbrace{\mathrm{KL}(p_{T}\Vert\hat{q}_{0})}_{\textnormal{Initialization error}}+\;\epsilon_{\text{score}}^2+\underbrace{hd\log(T/\delta)+h^2d^4T^2+Ch^3d^7T^3}_{\textnormal{Discretization error}}.
    \end{aligned}$$
    Here, the constants hidden by $\lesssim$ (i.e., $=O(\cdot)$) are independent of the manifold dimension $d$, but may depend on the geometric constants in Assumption \ref{assumption:manifold}.
    \label{thm:sde}
\end{theorem}

Theorem \ref{thm:sde} decomposes the distributional error into three components: i) the \emph{initialization error}, which arises from the discrepancy between the initial distribution of the sampling process and the desired initial distribution of the reverse process; ii) the \emph{score estimation error}, which quantifies the accuracy of the estimated score; and iii) the \emph{discretization error}, which arises from freezing the vector field over each time interval $[t_k,t_{k+1}]$. The first two terms do not depend on the step-size parameter $h$, and indicates the error we can achieve as the discretization becomes arbitrarily fine. In contrast, the discretization error vanishes as $h\to0$, and its dependence on  $h$ determines the convergence rate of the sampling scheme. We show the full proof of Theorem \ref{thm:sde} in Appendix \ref{appendix:sde}. 

If we use the step sizes specified in Theorem \ref{thm:sde}, the total number of discretization steps satisfies
$N=\Theta\left(h^{-1}\log(T/\delta)\right)$. Suppose further that $\mathrm{KL}(p_{T}\Vert\hat{q}_{0})+\epsilon_{\text{score}}^2=O(\epsilon^2)$. Then Theorem \ref{thm:sde} implies that, to achieve an $O(\epsilon^2)$ KL error (and hence an $O(\epsilon)$ TV error) for sufficiently small $\epsilon>0$, it suffices to choose $$N=O\left((\epsilon^{-2}d\log(T/\delta)+\epsilon^{-1}d^2T+\epsilon^{-2/3}d^{7/3}T)\log(T/\delta)\right)=O\left(\epsilon^{-2}d\log^2(T/\delta)\right).$$
For comparison, the analysis of \citet{xu2026polynomial} yields $N=O\left(\frac{d^6T^2}{\delta^3\epsilon^2}\right)$ in the setting considered here\footnote{We obtain this corresponding step-size bound from Lemma 2 of \citet{xu2026polynomial}, which applies precisely to the setting considered in this section.}. Our result substantially improves the dependence on the dimension and replaces the polynomial dependence on $1/\delta$ with a polylogarithmic dependence. Moreover, up to logarithmic factors, the number of discretization steps has the same dependence on $d$ and $\epsilon$ as the as the convergence guarantee for diffusion models in Euclidean space established by \citet{benton2024dlinear}. In particular, for compact manifolds with nonnegative Ricci curvature, our analysis shows that the dimension dependence can be reduced to essentially linear. An overview of the argument leading to this nearly linear dependence on $d$ is provided in Appendix \ref{appendix:sketch}.

\paragraph{Remark on General Compact Manifolds.} Our analysis for Theorem \ref{thm:sde} can be applied without the nonnegative Ricci curvature assumption. We provide the discussions in Appendix \ref{appendix:ricci}, where we modify Assumption \ref{assumption:manifold} so that the Ricci curvature is instead bounded below by $-C_R$ for some constant $C_R>0$. This assumption always holds for compact manifolds, and the definition of Ricci curvature directly implies that $C_R\le C_{\text{Rm}}d$. In the worst case $C_R=C_{\text{Rm}}d$, our analysis still implies a bound of $N=\tilde{O}\left(\frac{d^2T^2}{\epsilon^2}\right)$ on the number of discretization steps, which improves the bound $O\left(\frac{d^6T^2}{\delta^3\epsilon^2}\right)$ provided in \citep{xu2026polynomial} by $d^4$ in dimension for general compact smooth manifolds.
\section{Convergence with Geodesic Random Walk}

\label{sec:bm}

In the previous section, we analyzed the distributional error between the true reverse process (\ref{eq:reverse-riemannian}) and the approximated reverse process (\ref{eq:reverse-riemannian-estimate}) under an exact Brownian-motion simulation oracle. We now analyze the additional error introduced by implementing the oracle using geodesic random walks. Specifically, we perform $n_{\mathrm{bm}}$ geodesic random walk steps within each interval of Algorithm \ref{alg:framework}. The procedure is given in Algorithm \ref{alg:grw}. When $n_{\mathrm{bm}}=1$, this reduces to the standard geodesic random-walk scheme. Increasing $n_{\mathrm{bm}}$ improves the accuracy of Brownian-motion simulation without increasing the number of score evaluations. As we show below, its convergence rate differs from that of the score discretization error analyzed in the previous section.

\begin{algorithm}[htbp]
\caption{Implementing $\mathcal{O}(b,x,h)$ using $n_{\text{bm}}$ geodesic random walks.}
\label{alg:grw}
\begin{algorithmic}[1]

\Require $b$, $x\in\mathcal{M}$, $h\ge0$, $n_{\text{bm}}\in\mathbb{Z}_+$.
\State $h'\gets h/n_{\text{bm}}$.
\State $x_0\gets x$.
\For{$k=0,1,\dots,n_{\text{bm}}-1$}
    \State Sample Gaussian random vector $w_k$ in $T_{x_k}\mathcal{M}$.
    \State $v_k\gets h'\cdot b(x_k)+\sqrt{h'}w_k$.
    \State $x_{k+1}\gets\exp_{x_k}\left(v_k\right)$.
\EndFor
\State $\mathcal{O}(b,x,h)\gets x_{n_{\text{bm}}}$.

\end{algorithmic}
\end{algorithm}

To analyze this approximation, we require the following additional regularity assumptions.

\begin{assumption}
    In addition to Assumption \ref{assumption:manifold}, assume that $\max\{\Vert\nabla^3R\Vert,\Vert\nabla^4R\Vert\}\le C_{\mathrm{Rm}}'$.
    \label{assumption:manifold2}
\end{assumption}

Similarly as in Assumption \ref{assumption:manifold}, we assume that $C_{\mathrm{Rm}}'$ is independent of the dimension $d$. This additional regularity is needed to control higher-order derivatives of the metric tensor $g$.

\begin{assumption}
    For every $k=0,\dots,N-1$ and $x\in\mathcal{M}$, $\Vert\hat{s}_{t_k}(x)\Vert\le C_{\mathrm{score}}\left(\frac{D}{T-t_k}+\sqrt{\frac{d}{T-t_k}}\right)$,
    where $C_{\mathrm{score}}$ is independent of $d$, and $D$ is the diameter of $\mathcal{M}$.
    \label{assumption:score-estimate2}
\end{assumption}

Assumption \ref{assumption:score-estimate2} imposes a uniform bound on the estimated score. By Lemma \ref{lem:score-bound}, the true score under nonnegative Ricci curvature satisfies the same bound. In this case, we may clip $\hat{s}_{t_k}$ to this range without increasing its pointwise estimation error, so Assumption \ref{assumption:score-estimate} will remain valid after clipping.

We are now ready to state the convergence guarantee for the geodesic random walk approximation. In Algorithm \ref{alg:framework}, we continue to use (\ref{eq:extended-score}) to define the velocity field $\tilde{s}_k(\cdot)$. Let $\hat{Y}_{T-\delta}$ denote the solution to SDE (\ref{eq:reverse-riemannian-estimate}), and let $\tilde{Y}_{T-\delta}$ denote the output of Algorithm \ref{alg:framework} when the Brownian-motion simulation oracle is implemented by Algorithm \ref{alg:grw}. We assume that both processes are initialized from the same distribution, $\hat{Y}_0\sim\hat{q}_0$. The following theorem bounds the additional error introduced by the geodesic random-walk simulation.

\begin{theorem}
    Let $\hat{q}_{T-\delta}$ and $\tilde{q}_{T-\delta}$ denote the distributions of $\hat{Y}_{T-\delta}$ and $\tilde{Y}_{T-\delta}$, respectively. Let $h':=h/n_{\text{bm}}$ denote the inner step size scale, where $h>0$ is the scale parameter defined in Theorem \ref{thm:sde}. Under Assumptions \ref{assumption:manifold}, \ref{assumption:manifold2} and \ref{assumption:score-estimate2}, if $h'$ is sufficiently small, then
    $$\mathrm{TV}(\hat{q}_{T-\delta},\tilde{q}_{T-\delta})^2\lesssim h'\left(d^4T+\delta^{-2}D^4+D^2d^3\log(T/\delta)\right),$$
    where the constants hidden by $\lesssim$ are independent of the dimension $d$.
    \label{thm:bm}
\end{theorem}

The detailed proof of Theorem \ref{thm:bm} is provided in Appendix \ref{appendix:bm}. Although we use a parametrix-based analysis similar to that of \citet{xu2026polynomial}, our refined estimates lead to a substantially improved dependence on the dimension $d$.

Theorem \ref{thm:bm}, shows that the squared TV error scales linearly with the inner step size $h'$, as in the discretization error in Theorem \ref{thm:sde}. However, its dependence on the dimension is substantially worse. This indicates that accurately simulating drifted Brownian motion may require more computational steps than discretizing the score function alone. Combining the two error bounds yields the following result, whose proof is given in\ref{appendix:proof-of-num-steps}.

\begin{corollary}
    Suppose Assumptions \ref{assumption:manifold}, \ref{assumption:score-estimate}, \ref{assumption:manifold2} and \ref{assumption:score-estimate2} hold. Assume further that $\delta=\Omega(d^{-2})$, $\mathrm{TV}(\hat{q}_0,p_T)+\epsilon_{\text{score}}=O(\epsilon)$. Then, for sufficiently small $\epsilon>0$, under the step-size schedule of Theorem \ref{thm:sde}, it suffices to choose $$N=O\left(\epsilon^{-2}d\log^2(T/\delta)\right),\quad N'=O\left(\epsilon^{-2}d^4T\log(T/\delta)\right)$$
    to ensure $\mathrm{TV}(p_\delta,\tilde{q}_{T-\delta})=O(\epsilon)$. Here, $N$ is the number of evaluations of score function, and $N'=Nn_{\text{bm}}$ is the total number of geodesic random-walk steps.
    \label{cor:num-of-steps}
\end{corollary}

Corollary \ref{cor:num-of-steps} provides complexity bounds for both score evaluations and geodesic random-walk steps. Since evaluating the neural-network score is typically substantially more expensive than simulating a single geodesic random-walk step, it can be computationally advantageous to perform multiple random-walk steps between consecutive score evaluations. Thus, the corollary provides a theoretical justification for using multiple Brownian-motion simulation steps per score evaluation.

\paragraph{Remark on the time horizon $T$.} The initial distribution $\hat{q}_0$ is typically chosen to be the uniform distribution $u_{\mathcal{M}}$ with respect to the Riemannian volume measure. Therefore, $T$ must be sufficiently large to ensure that that $\mathrm{TV}(p_T,u_{\mathcal{M}})=O(\epsilon)$. As shown by \citep{xu2026polynomial}, it suffices to choose $T=O\left(\lambda_1^{-1}(d\log d+\log(1/\epsilon))\right)$, where $\lambda_1$ is the smallest positive eigenvalue of $-\Delta$. For general manifold, obtaining an explicit lower bound on $\lambda_1$ can be difficult. However, for the $d$-dimensional unit sphere, it is known that $\lambda_1=d$, giving $T=O(\log d+d^{-1}\log(1/\epsilon))$.
\section{Conclusion}

In this work, we developed a sharp convergence analysis for Riemannian diffusion models. We introduced a general sampling framework (Section \ref{sec:framework}) that explicitly separates the sampling error into score discretization error and Brownian-motion simulation error. Under an exact Brownian-motion simulation oracle, our analysis shows that $\tilde{O}(\epsilon^{-2}d)$ score evaluations suffice to achieve an $\epsilon^2$ KL error (Theorem \ref{thm:sde}), substantially improving existing convergence guarantees and matching the nearly linear dimension dependence known for Euclidean diffusion models. We further analyzed the Brownian-motion simulation error of geodesic random walks and showed that, from the resulting error trade-off, performing multiple geodesic random-walk steps per score evaluation can be computationally advantageous.

Several directions remain open for future work. First, our sampling framework is sufficiently general to accommodate more sophisticated sampling schemes, leaving room for further improvements through better choices of the vector-field extension or more accurate and efficient Brownian-motion simulation methods. Second, our TV guarantees are obtained via Pinsker's inequality from the KL bounds. Analyses of Euclidean diffusion models suggest that a direct analysis of TV error may yield sharper convergence rates, and it would be interesting to investigate whether similar improvements are possible on Riemannian manifolds. Finally, our analysis focuses on the SDE formulation of diffusion sampling, while the convergence properties of the corresponding probability-flow ODE on Riemannian manifolds remain largely unexplored.



\bibliography{iclr2027_conference}
\bibliographystyle{dilab_ref}

\makeappendixtoc
\appendix
\section*{Notations}

Throughout the proof, we use $C$ to represent constants that are independent of $d$, and for notational simplicity, the value of $C$ may change from line to line. We use greek alphabet (i.e., $\alpha,\beta,\eta,\gamma$, etc.) to denote the indices of tensors in local coordinates, and we assume Einstein’s summation convention.

We use $\mathcal{M}$ to denote the Riemannian manifold, and $g$ represents the Riemannian metric. We write $\sqrt{g}:=\sqrt{\det(g)}$ as the volume unit. $\int_{\mathcal{M}}f(x)\cdot \sqrt{g(x)}dx$ denotes the integration on $\mathcal{M}$ with respect to the volume measure. We also simply write $\int_{\mathcal{M}}f(x)dx$ to omit the volume unit.

We use $T\mathcal{M}$ to represent the tangent bundle of $\mathcal{M}$, and $T_x\mathcal{M}$ to denote the tangent space at $x\in\mathcal{M}$. The basis for $T_x\mathcal{M}$ in the local coordinate representation is denoted by $\{\partial_\alpha\}_{\alpha=1}^d$ or $\{e_\alpha\}_{\alpha=1}^d$. For a smooth function $f(\cdot)$ on the manifold, we use $\nabla f$ to represent its Riemannian gradient, which satisfies $$\forall x\in\mathcal{M},\forall v\in T_x\mathcal{M},\quad\langle\nabla f(x),v\rangle=\dif f(x)(v).$$
We use Christoffel symbols
$$\Gamma_{\alpha\beta}^\gamma:=\frac{1}{2}g^{\gamma\eta}\left(\partial_{\alpha}g_{\beta\eta}+\partial_{\beta}g_{\alpha\eta}-\partial_{\eta}g_{\alpha\beta}\right),$$
The Levi-civita connection is given by $$\nabla_\alpha V(=\nabla_{e_\alpha}V)=\partial_\alpha v^\beta e_\beta+\Gamma_{\alpha\beta}^\gamma v^\beta e_\gamma,$$
where $V(\cdot)=v^\beta e_\beta$ is a vector field on $\mathcal{M}$. The Levi-civita connection can be extended for general tensors, see e.g., Chapter 4 of \citep{jost2005riemannian}.

For Riemannian curvature, we use the convension $$R(X,Y)Z=\nabla_X\nabla_YZ-\nabla_Y\nabla_XZ-\nabla_{[X,Y]}Z$$
for vector fields $X$, $Y$ and $Z$, where $[\cdot,\cdot]$ denotes the Lie bracket
$$[X,Y]=\left[X^\beta\partial_\beta Y^\alpha -Y^\beta\partial_\beta X^\alpha\right]e_\alpha,\quad X=X^\alpha e_\alpha,Y=Y^\alpha e_\alpha.$$
We use the convension 
$$\Vert R\Vert\le\sup_{\Vert X\Vert=\Vert Y\Vert=\Vert Z\Vert=1}\Vert R(X,Y)Z\Vert$$
for the operator norm, which is well-defined because the value of $R(X,Y)Z$ at a point $x\in\mathcal{M}$ only depends on the value of $X,Y,Z$ at $x$. The norm is defined similarly for the derivatives. Moreover, the Ricci curvature is given by $$\mathrm{Ric}(X,Y)=g^{\alpha\beta}\langle R(X,e_\alpha)e_\beta,Y\rangle,$$
i.e., the trace of the Riemannian curvature tensor.

The Laplace-Beltrami operator on $\mathcal{M}$ for smooth functions is defined by
$$\Delta f=\mathrm{div}(\nabla f)=\frac{1}{\sqrt{g}}\partial_\alpha\left(\sqrt{g}g^{\alpha\beta}\partial_\beta f\right)=g^{\alpha\beta}\left(\partial_{\alpha\beta}-\Gamma_{\alpha\beta}^\gamma\partial_{\gamma}\right)f.$$
Note that some mathematical books have also used the definition $-\Delta$, to ensure nonnegative eigenvalues. For a general tensor $T$, we use the extension given by $$\Delta T=g^{\alpha\beta}\left(\nabla_{\alpha}\nabla_{\beta}-\Gamma_{\alpha\beta}^\gamma\nabla_{\gamma}\right)T.$$

For any linear operator $T:H_1\to H_2$ between Hilbert spaces, we define its Hilbert-Schmidt norm by $$\Vert T\Vert_{\text{HS}}=\sqrt{\sum_{i}\Vert T e_i\Vert^2},\quad\{e_i\}\text{ is an orthonormal basis of }H_1.$$
Clearly, when $H_1$ and $H_2$ are Euclidean spaces, the Hilbert-Schimidt norm is exactly the Frobenius norm of matrix representation of $T$ in orthonormal bases.

\section{Improved Analysis to Achieve $d$-linear Dependence}

\label{appendix:sketch}

In Theorem \ref{thm:sde}, the dominant term for discretization error is $hd\log(T/\delta)$, which is linear in $d$. This eventually produces the $d$-linear convergence rate for the sampling process. In this section, we sketch through some key technical improvements that lead to this bound.

Our goal is to compare the generated distributions given by SDEs (\ref{eq:reverse-riemannian}) and (\ref{eq:reverse-riemannian-estimate}). Since both equations share the same Brownian motion part, a standard Girsanov-based approach indicates that
$$\mathrm{KL}(P\Vert\hat{P})=\mathrm{KL}(p_{T}\Vert\hat{q}_{0})+\frac{1}{2}\sum_{k=0}^{N-1}\int_{t_k}^{t_{k+1}}\mathbb{E}_{Y\sim P}\left[\Vert\nabla\log p_{T-t}(Y_t)-\tilde{s}_k(Y_t)\Vert^2\right]\dif t,$$
where $P$ and $\hat{P}$ are path measures of $(Y_t)_{t\ge0}$ and $(\hat{Y}_t)_{t\ge0}$, respectively. By extracting the score matching error and apply the Definition \ref{definition:bm} of geometric Brownian motion (i.e., It\^o's rule) on the score difference, where we skip the detailed calculations here, the KL divergence are eventually bounded by
$$\mathrm{KL}(P\Vert\hat{P})\lesssim\mathrm{KL}(p_{T}\Vert\hat{q}_{0})+\epsilon_{\text{score}}^2+\sum_{k=0}^{N-1}h_k\int_{t_k}^{t_{k+1}}\mathbb{E}_{Y\sim P}\left[\Vert\nabla^2\log p_{T-t}(Y_t)\Vert^2_{\text{HS}}\right]\dif t+o(h).$$
In the above error bound, the dominating term consists of the expectation of squared Hilbert-Schmidt norm for Hessians of log density. However, until now, it is still insufficient to obtain the $d$-linear error bound, because even for the case of Euclidean diffusion model, the worst-case bound for $\mathbb{E}_{Y\sim P}\left[\Vert\nabla^2\log p_{T-t}(Y_t)\Vert^2_{\text{HS}}\right]$ is still of order $O(d^2)$. To further reduce the degree of $d$, we prove two key inequalities under the assumption of nonnegative Ricci curvature:
\begin{equation}
    (T-t)\mathbb{E}[\Vert\nabla\log p_{T-t}(Y_t)\Vert^2_{\text{HS}}]\le\frac{\dif}{\dif t}(T-t)\mathbb{E}[\Vert\nabla\log p_{T-t}(Y_t)\Vert^2]+\mathbb{E}[\Vert\nabla\log p_{T-t}(Y_t)\Vert^2];
    \label{eq:hs-bound}
\end{equation}
\begin{equation}
    \mathbb{E}_{x\sim p_t}\left[\Vert\nabla\log p_t(x)\Vert^2\right]\le\frac{Cd}{t},\quad C\text{ is a universal constant}.
    \label{eq:score-bound}
\end{equation}
Therefore, apply inequalities (\ref{eq:hs-bound}) and (\ref{eq:score-bound}), and note that our selected step size schedule satisfies that $h_k=O(h(T-t))$ for $t\in[t_{k},t_{k+1}]$,
$$\begin{aligned}
    &\sum_{k=0}^{N-1}h_k\int_{t_k}^{t_{k+1}}\mathbb{E}_{Y\sim P}\left[\Vert\nabla^2\log p_{T-t}(Y_t)\Vert^2_{\text{HS}}\right]\dif t\\
    &\le Ch\sum_{k=0}^{N-1}\int_{t_k}^{t_{k+1}}(T-t)\mathbb{E}_{Y\sim P}\left[\Vert\nabla^2\log p_{T-t}(Y_t)\Vert^2_{\text{HS}}\right]\dif t\\
    &\le Ch\left(\delta\mathbb{E}[\Vert\nabla\log p_\delta(X_\delta)\Vert^2]+\int_\delta^T\mathbb{E}[\Vert\nabla\log p_t(X_t)\Vert^2]\dif t\right)\\
    &\le Ch\left(d+\int_\delta^T\frac{d}{t}\dif t\right)=O\left(hd\log(T/\delta)\right).
\end{aligned}$$
This is exactly the bound claimed in Theorem \ref{thm:sde}.

\section{Proof of Theorem \ref{thm:sde}}

\label{appendix:sde}

Recall the true backward process is given by $$\dif Y_t=\nabla\log p_{T-t}(Y_t)\dif t+\dif B_t,$$
and we have defined the estimated backward process by $$\dif\hat{Y}_t=\tilde{s}_t(\hat{Y}_t)\dif t+\dif B_t,$$
where the drift $\tilde{s}_t$ is defined by $$\tilde{s}_t(\exp_{\hat{Y}_{t_k}}(v)):=\tilde{s}_k(\exp_{\hat{Y}_{t_k}}(v)):=\eta_\omega(|v|) \dif\left[\exp_{\hat{Y}_{t_k}}(v)\right]\left(\hat{s}_t(\hat{Y}_{t_k})\right),\quad t\in[t_k,t_{k+1}),$$
which depends on the process $(\hat{Y}_{s})_{t_0\le s\le t}$ before $t$. Here, we select the scale parameter as $\omega:=\frac{c}{\sqrt{d}}$ where the small constant $c>0$ only depends on $r_i$ and $C_{\text{Rm}}$, and the cut-off function $\eta_{\omega}$ satisfies 
$$\eta_\omega(r)=\begin{cases}
    1,&r\le\omega/2,\\
    0,&r\ge\omega,
\end{cases}\quad\eta_{\omega}'(r)=O(\omega^{-1}),\quad\eta''_{\omega}(r)=O(\omega^{-2}).$$
The function $\tilde{s}_t$ expand the estimated score function at $Y_{t_k}$ to the local geodesic ball. Note that according to the definition, the vector $\dif\left[\exp_{\hat{Y}_{t_k}}(v)\right]\left(\hat{s}_{t}(\hat{Y}_{t_k})\right)$ is literally a constant vector in the local normal coordinate representation at $\hat{X}_{t_k}$. This can be viewed as a counterpart of Euclidean diffusion model case, where we fix the score function and run the Euler-Maruyama discretization to sample data.

Also, recall that in this proof, our eventual goal is to bound the KL divergence between the distribution of $Y_{t_N}$ and $\hat{Y}_{t_N}$. 

\subsection{Overall Framework}

We adapt the same high-level idea of \citet{xu2026polynomial}, which applies the Girsanov's theorem to compute the KL divergence between processes. Let $P$ and $\hat{P}$ be the path measure of $(Y_t)_{t_0\le t\le t_N}$ and $(\hat{Y}_t)_{0\le t\le t_N}$, then, according to the standard Girsanov-based framework \citep{chen2023sampling,debortoli2022riemannian}, 
\begin{equation}
    \mathrm{KL}(P\Vert\hat{P})=\mathrm{KL}(P_{t_0}\Vert\hat{P}_{t_0})+\frac{1}{2}\int_{t_0}^{t_N}\mathbb{E}_{Y\sim P}\left[\Vert\nabla\log p_{T-t}(Y_t)-\tilde{s}_t(Y_t)\Vert^2\right]\dif t.
    \label{eq:girsanov}
\end{equation}

We now need to compare the two drifts $\nabla\log p_{T-t}(Y_t)$ and $\tilde{s}_t(Y_t)$. We first bridge the two drifts through an intermediate function defined by $$\tilde{s}_t^*(\exp_{Y_{t_k}}(v)):=\eta_\omega(|v|)\dif\left[\exp_{Y_{t_k}}(v)\right]\left(\nabla\log p_{T-t_k}(Y_{t_k})\right).$$
That is, $\tilde{s}^*_t(\cdot)$ expands the true score function instead of the estimated score function. Clearly, $$\Vert\nabla\log p_{T-t}(Y_t)-\tilde{s}_t(Y_t)\Vert^2\le2\underbrace{\Vert\nabla\log p_{T-t}(Y_t)-\tilde{s}_t^*(Y_t)\Vert^2}_{\text{Discretization Error}}+2\underbrace{\Vert\tilde{s}_t^*(Y_t)-\tilde{s}_t(Y_t)\Vert^2}_{\text{Score Estimation Error}}.$$

According to Lemma \ref{lem:local-g}, as long as we take $\omega=\frac{c}{\sqrt{d}}$ for a small global constant depending on $C_{\text{Rm}}$ and $r_i$, the matrix representation of Riemannian metric $g$ inside the local normal coordinate chart with radius $\omega$ satisfies $\Vert g\Vert\le C$. Therefore, by that $\eta_\omega(\cdot)\le 1$, $$\Vert\tilde{s}^*_t(Y_t)-\tilde{s}_t(Y_t)\Vert^2\le C\Vert\hat{s}_{t_k}(Y_{t_k})-\nabla\log p_{T-t_k}(Y_{t_k})\Vert^2.$$
Applying Assumption \ref{assumption:score-estimate},
$$\int_{t_0}^{t_N}\mathbb{E}_{Y\sim P}\left[\Vert\tilde{s}_t^*(Y_t)-\tilde{s}_t(Y_t)\Vert^2\right]\dif t\lesssim\sum_{k=0}^{N-1}h_k\mathbb{E}_{y\sim p_{T-t_k}}\left[\Vert\hat{s}_{t_k}(y)-\nabla\log p_{T-t_k}(y)\Vert^2\right]\lesssim\epsilon_{\text{score}}^2.$$
Therefore, in order to prove Theorem \ref{thm:sde}, the only thing left is to control $$\text{Discretization Error}=2\int_{t_0}^{t_N}\mathbb{E}\left[\Vert\nabla\log p_{T-t}(Y_t)-\tilde{s}_t^*(Y_t)\Vert^2\right]\dif t.$$

\subsection{It\^o's Rule for Discretization Error}

In this section, we explicitly compute the formula for discretization error by applying Ito's lemma on $\Vert\nabla\log p_{T-t}(Y_t)-\tilde{s}_t^*(Y_t)\Vert^2$ in the interval $t\in[t_k,t_{k+1})$, where we can view $Y_{t_k}$ as a fixed point. To do so, we must compute the time and spacial derivatives of $\Vert\nabla\log p_{T-t}(Y_t)-\tilde{s}_t^*(Y_t)\Vert^2$. First, the density function $p_t(\cdot)$ satisfies the heat equation
$$\partial_tp_t(x)=\frac{1}{2}\Delta p_t(x).$$
Hence, 
\begin{equation}
    \begin{aligned}
        \partial_t\nabla\log p_t(x)&=\nabla\left(\frac{\partial_t p_t(x)}{p_t(x)}\right)=\frac{1}{2}\nabla\left(\frac{\Delta p_t(x)}{p_t(x)}\right)\\
        &\overset{\text{(a)}}{=}\frac{1}{2}\nabla\left(\Delta\log p_t(x)+\Vert\nabla\log p_t(x)\Vert^2\right)\\
        &=\frac{1}{2}\nabla\Delta\log p_t(x)+\nabla^2\log p_t(x)\cdot\nabla\log p_t(x),
    \end{aligned}
    \label{eq:dt-score}
\end{equation}
where step (a) comes from the equality $$\Delta\log f=\mathrm{div}(\nabla\log f)=\mathrm{div}\left(\frac{\nabla f}{f}\right)=\frac{\Delta f}{f}-\frac{\Vert\nabla f\Vert^2}{f^2}.$$
Moreover, $\partial_t\tilde{s}_t^*(\cdot)=0$ for $t\in[t_k,t_{k+1})$, as the expanded drift is actually independent of $t$. Therefore, using that $\dif Y_t=\nabla\log p_{T-t}(Y_t)+\dif B_t$,
$$\begin{aligned}
    \dif\Vert\nabla\log p_{T-t}(Y_t)-\tilde{s}_t^*(Y_t)\Vert^2
    &=\partial_t\Vert\nabla\log p_{T-t}(Y_t)-\tilde{s}_t^*(Y_t)\Vert^2\dif t+\frac{1}{2}\Delta\Vert\nabla\log p_{T-t}(Y_t)-\tilde{s}_t^*(Y_t)\Vert^2\dif t\\
    &\quad+\nabla\Vert\nabla\log p_{T-t}(Y_t)-\tilde{s}_t^*(Y_t)\Vert^2\cdot(\nabla\log p_{T-t}(Y_t)\dif t+\dif B_t)\\
    &\overset{\text{(a)}}{=}2(\nabla\log p_{T-t}(Y_t)-\tilde{s}_t^*(Y_t))\\
    &\quad\cdot\bigg[\partial_t\nabla\log p_{T-t}(Y_t)\dif t+\nabla^2\log p_{T-t}(Y_t)\cdot(\nabla\log p_{T-t}(Y_t)\dif t+\dif B_t)\\
    &\qquad-\nabla\tilde{s}_t^*(Y_t)\cdot(\nabla\log p_{T-t}(Y_t)+\dif B_t)\bigg]\\
    &\quad+(\nabla\log p_{T-t}(Y_t)-\tilde{s}_t^*(Y_t))\cdot(\Delta\nabla\log p_{T-t}(Y_t))-\Delta\tilde{s}^*_t(Y_t))\dif t\\
    &\quad+\Vert\nabla^2\log p_{T-t}(Y_t)-\nabla\tilde{s}_t^*(Y_t)\Vert_{\text{HS}}^2\dif t\\
    &\overset{\text{(b)}}{=}(\nabla\log p_{T-t}(Y_t)-\tilde{s}_t^*(Y_t))\\
    &\quad\cdot\bigg[(\Delta\nabla-\nabla\Delta)\log p_{T-t}(Y_t)\dif t+\nabla^2\log p_{T-t}(Y_t)\cdot \dif B_t\\
    &\qquad-\nabla\tilde{s}^*_t(Y_t)\cdot(\nabla\log p_{T-t}(Y_t)\dif t+\dif B_t)-\Delta\tilde{s}_t(Y_t)\bigg]\\
    &\quad+\Vert\nabla^2\log p_{T-t}(Y_t)-\nabla\tilde{s}_t^*(Y_t)\Vert_{\text{HS}}^2\dif t\\
    &\overset{\text{(c)}}{=}(\nabla\log p_{T-t}(Y_t)-\tilde{s}_t^*(Y_t))\\
    &\quad\cdot\bigg[\mathrm{Ric}^{\#}(\nabla\log p_{T-t}(Y_t),\cdot)\dif t+\nabla^2\log p_{T-t}(Y_t)\cdot \dif B_t\\
    &\qquad-\nabla\tilde{s}^*_t(Y_t)\cdot(\nabla\log p_{T-t}(Y_t)dt+dB_t)-\Delta\tilde{s}_t(Y_t)\bigg]\\
    &\quad+\Vert\nabla^2\log p_{T-t}(Y_t)-\nabla\tilde{s}_t^*(Y_t)\Vert_{\text{HS}}^2\dif t.
\end{aligned}$$
Here, (a) comes from the equality $$\Delta\Vert V\Vert^2=2\langle\Delta V,V\rangle+2\Vert\nabla V\Vert_{\text{HS}}^2,$$
(b) applies Equation \ref{eq:dt-score}, and (c) uses the Bochner’s identity $$(\Delta\nabla-\nabla\Delta)f=\mathrm{Ric}^{\#}(\nabla f,\cdot),$$
and $\mathrm{Ric}^{\#}(V,\cdot)$ represents the vector $V'$ satisfying $\langle V',\cdot\rangle=\mathrm{Ric}(V,\cdot)$, i.e., raising one index in the Ricci curvature. 

\subsection{Controlling Local Part of Discretization Error}

Let $\tau_k$ be the first time when $\rho(Y_t,Y_{t_k})\ge r:=\omega/2$. Then, by taking expectation and apply the stopping time theorem, 
$$\begin{aligned}
    &\qquad\mathbb{E}\Vert\nabla\log p_{T-t\land\tau_k}(Y_{t\land\tau_k})-\tilde{s}_{t\land\tau_k}^*(Y_{t\land\tau_k})\Vert^2\\
    &=\int_{t_k}^t\mathbb{E}\bigg[1_{u\le\tau_k}\cdot(\nabla\log p_{T-u}(Y_u)-\tilde{s}_u^*(Y_u))\\
    &\cdot\bigg(\mathrm{Ric}^{\#}(\nabla\log p_{T-u}(Y_u),\cdot)-\nabla\tilde{s}^*_u(Y_u)\cdot\nabla\log p_{T-u}(Y_u)-\Delta\tilde{s}_u^*(Y_u)\bigg)\bigg]\dif u\\
    &\quad+\int_{t_k}^t\mathbb{E}\left[1_{u\le\tau_k}\cdot\Vert\nabla^2\log p_{T-u}(Y_u)-\nabla\tilde{s}_t^*(Y_u)\Vert_{\text{HS}}^2\right]\dif u\\
    &\le\int_{t_k}^t\frac{1}{h_k}\cdot \mathbb{E}\Vert\nabla\log p_{T-u}(Y_u)-\tilde{s}^*_u(Y_u)\Vert^2\dif u\\
    &+\int_{t_k}^t\frac{h_k}{4}\mathbb{E}\left(1_{u\le\tau_k}\Big\Vert\mathrm{Ric}^\#(\nabla\log p_{T-u}(Y_u),\cdot)-\nabla\tilde{s}^*_u(Y_u)\cdot\nabla\log p_{T-u}(Y_u)-\Delta\tilde{s}^*_u(Y_u)\Big\Vert^2\right)\dif u\\
    &\quad+\int_{t_k}^t\mathbb{E}\left[1_{u\le\tau_k}\cdot\Vert\nabla^2\log p_{T-u}(Y_u)-\nabla\tilde{s}_t^*(Y_u)\Vert_{\text{HS}}^2\right]\dif u,
\end{aligned}$$
where the last inequality directly uses $xy\le\frac{x^2+y^2}{2}$. Since $\int_{t_k}^t\frac{1}{h_k}\le 1$, using Gr\"onwall's inequality, 
$$\begin{aligned}
    &\qquad\mathbb{E}\Vert\nabla\log p_{T-t\land\tau_k}(Y_{t\land\tau_k})-\tilde{s}_{t\land\tau_k}^*(Y_{t\land\tau_k})\Vert^2\\
    &\le C\int_{t_k}^th_k\mathbb{E}\left(1_{u\le\tau_k}\Vert\mathrm{Ric}^\#(\nabla\log p_{T-u}(Y_u),\cdot)-\nabla\tilde{s}^*_u(Y_u)\cdot\nabla\log p_{T-u}(Y_u)-\Delta\tilde{s}^*_u(Y_u)\Vert^2\right)\dif u\\
    &\quad+C\int_{t_k}^t\mathbb{E}\Vert\nabla^2\log p_{T-u}(Y_u)\Vert^2_{\text{HS}}\dif u+C\int_{t_k}^t\mathbb{E}[\Vert\nabla\tilde{s}^*_t(Y_u)\Vert_{\text{HS}}^21_{u\le\tau_k}]\dif u.
\end{aligned}$$

Now we calculate the term that are related to the locally fixed score $\tilde{s}^*$ within radius $r:=\omega/2$. For $x\in\mathcal{M}$ with $\rho(x,Y_{t_k})\le r$, consider any local normal coordinate chart at $Y_{t_k}$, since the coordinate representation of $\tilde{s}_t^*(x)$ is constant inside the local coordinate chart, we have
$$\begin{aligned}
    \nabla_\alpha\tilde{s}^*_t(x)&=\Gamma_{\alpha\gamma}^\beta\nabla\log p_{T-t_k}(Y_{t_k})^\gamma\partial_\beta,\\
    \Delta\tilde{s}^*_t(x)&=g^{\alpha\beta}(\nabla_\alpha\nabla_\beta-\nabla_{\nabla\alpha(\partial_\beta)})\tilde{s}^*_t(x)\\
    &=g^{\alpha\beta}\left[\nabla_\alpha(\Gamma_{\beta\zeta}^\xi\nabla\log p_{T-t_k}(x)^\zeta\partial_\xi)-\Gamma_{\alpha\beta}^\xi\nabla_\xi\tilde{s}_t^*(x)\right]\\
    &=g^{\alpha\beta}\left[\partial_\alpha\Gamma_{\beta\zeta}^\gamma+\Gamma_{\beta\zeta}^\xi\Gamma_{\alpha\xi}^{\gamma}-\Gamma_{\alpha\beta}^\xi\Gamma_{\xi\zeta}^\gamma\right]\nabla\log p_{T-t_k}(Y_{t_k})^\zeta\partial_\gamma.
\end{aligned}$$
Since Lemma \ref{lem:local-g} holds for arbitrary choice of local normal coordinates, we can eventually choose a normal coordinate so that $\nabla\log p_{T-t_k}(Y_{t_k})$ has only one nonzero component in the coordinate representation. Therefore, using Lemma \ref{lem:local-g},
$$\begin{aligned}
    \Vert\nabla\tilde{s}^*_u(Y_u)\cdot\nabla\log p_{T-u}(Y_u)\Vert&\le\Vert\nabla\log p_{T-t_k}(Y_{t_k})\Vert \cdot\sqrt{d}\max\Gamma_{\alpha\gamma}^\beta\cdot|\nabla\log p_{T-u}(Y_u)|_1\\
    &\le C\sqrt{d}\Vert\nabla\log p_{T-t_k}(Y_{t_k})\Vert\cdot\Vert\nabla\log p_{T-u}(Y_u)\Vert.
\end{aligned}$$
Moreover, Lemma \ref{lem:local-g} ensures that $\sum_{\alpha=1}^d\sum_{\beta=1}^dg^{\alpha\beta}=O(d)$ when $\omega\le\frac{c}{\sqrt{d}}$. So, in the same coordinate setting, for $\rho(x,Y_{t_k})\le r$,
$$\begin{aligned}
    \Vert\Delta\tilde{s}^*_t(x)\Vert&\le C\sqrt{d}\left[d\cdot\max\partial_\alpha\Gamma_{\beta\xi}^\gamma+d^2(\max\Gamma_{\alpha\beta}^\gamma)^2\right]\Vert\nabla\log p_{T-t_k}(Y_{t_k})\Vert\le C\Vert\nabla\log p_{T-t_k}(Y_{t_k})\Vert d^{3/2}.
\end{aligned}$$
In addition, the Ricci curvature is trivially bounded by $d\cdot C_{\text{Rm}}$, so we can bound $$\Vert\mathrm{Ric}^\#(\nabla\log p_{T-u}(Y_u),\cdot)\Vert\le Cd\Vert\nabla\log p_{T-u}(Y_u)\Vert.$$
For the Hilbert-Schmidt norm of $\nabla\tilde{s}_t(Y_u)$, apply the local coordinate argument once again, and use the coordinate representation of $\nabla\tilde{s}_t^*(x)$, we obtain that 
$$\begin{aligned}
    \Vert\nabla\tilde{s}^*_u(x)\Vert_{\text{HS}}^2&=\Vert\nabla\log p_{T-t_k}(Y_{t_k})\Vert^2\max_{\xi}g^{\alpha\beta}g_{\gamma\eta}\Gamma_{\alpha\xi}^\gamma\Gamma_{\beta\xi}^\eta\le  C\Vert\nabla\log p_{T-t_k}(Y_{t_k})\Vert^2d^2\cdot(\max\Gamma)^2\\
    &\le Cd^2\Vert\nabla\log p_{T-t_k}(Y_{t_k})\Vert^2\rho(x,Y_{t_k})^2.
\end{aligned}$$

Combining the results together,
$$\begin{aligned}
    &\mathbb{E}\Vert\nabla\log p_{T-t\land\tau_k}(Y_{t\land\tau_k})-\tilde{s}_{t\land\tau_k}^*(Y_{t\land\tau_k})\Vert^2\le C\int_{t_k}^th_k\bigg(d^2\mathbb{E}\Vert\nabla\log p_{T-u}(Y_u)\Vert^2\\&\qquad\quad+d\mathbb{E}\left[\Vert\nabla\log p_{T-u}(Y_u)\Vert^2\cdot\Vert\nabla\log p_{T-t_k}(Y_{t_k})\Vert^2\right]+\mathbb{E}\Vert\nabla\log p_{T-t_k}(Y_{t_k})\Vert^2 d^{3}\bigg)\dif u\\
    &\qquad\quad+C\int_{t_k}^t\mathbb{E}\Vert\nabla^2\log p_{T-u}(Y_u)\Vert^2_{\text{HS}}\dif u+Cd^2\int_{t_k}^t\mathbb{E}[\Vert\nabla\log p_{T-t_k}(Y_u)\Vert^2\rho(Y_u,Y_{t_k})^2]\dif u\\
    &\qquad\overset{\text{(a)}}{\le} C\int_{t_k}^th_k\left(\frac{d^3}{T-u}+\frac{d^4}{T-t_k}+d\mathbb{E}\Vert\nabla\log p_{T-u}(Y_u)\Vert^4+d\mathbb{E}\Vert\nabla\log p_{T-t_k}(Y_{t_k})\Vert^4\right)\dif u\\
    &\qquad\quad+C\int_{t_k}^t\mathbb{E}\Vert\nabla^2\log p_{T-u}(Y_u)\Vert^2_{\text{HS}}\dif u+Cd^2\int_{t_k}^t\sqrt{\mathbb{E}\Vert\nabla\log p_{T-t_k}(Y_{t_k})\Vert^4}\sqrt{\mathbb{E}\rho(Y_u,Y_{t_k})^4}\dif u\\
    &\qquad\overset{\text{(b)}}{\le} C\int_{t_k}^th_k\left(\frac{d^3}{T-u}+\frac{d^4}{T-t_k}+\frac{d^3}{(T-u)^2}+\frac{d^3}{(T-t_k)^2}\right)\dif u\\
    &\qquad\quad+C\int_{t_k}^t\mathbb{E}\Vert\nabla^2\log p_{T-u}(Y_{u})\Vert^2_{\text{HS}}\dif u+Cd^2\int_{t_k}^t\frac{d}{T-t_k}\cdot du\cdot\dif u,
\end{aligned}$$
where we apply Lemma \ref{lem:score-bound} for to obtain inequality (a), and apply Lemma \ref{lem:score-moment4} to obtain inequality (b). Note that during step (b), we used the fact that $(Y_{t})$ is exactly the time reversal of the forward process $(X_t)$, and the proof of Lemma \ref{lem:score-moment4} shows that the radial process $r(X_t)$ of Riemmannian Brownian motion (under nonnegative Ricci curvature) satisfies $\mathbb{E}[r(X_t)^4]\le O(d^2t^2)$.

Since we want to control the total discretization error, we must take integration over $t$. Thus,
$$\begin{aligned}
    \text{Discretization Error}&=C\int_{t_0}^{t_N}\mathbb{E}\left[\Vert\nabla\log p_{T-t}(Y_t)-\tilde{s}_t^*(Y_t)\Vert^2\right]\dif t\\
    &\le C\sum_{k=0}^{N-1}\int_{t_k}^{t_{k+1}}\mathbb{E}\Vert\nabla\log p_{T-t\land\tau_k}(Y_{t\land\tau_k})-\tilde{s}_{t\land\tau_k}^*(Y_{t\land\tau_k})\Vert^2\dif t\\
    &\quad+C\sum_{k=0}^{N-1}\int_{t_k}^{t_{k+1}}\mathbb{E}\left[\Vert\nabla\log p_{T-t}(Y_{t})-\tilde{s}^*_t(Y_t)\Vert^2\cdot 1_{\tau_k<t}\right]\dif t\\
    &\le C\sum_{k=0}^{N-1}h_k^3\left(\frac{d^4}{T-t_{k+1}}+\frac{d^3}{(T-t_{k+1})^2}\right)\\
    &\quad+\underbrace{C\sum_{k=0}^{N-1}h_k\int_{t_k}^{t_{k+1}}\mathbb{E}\Vert\nabla^2\log p_{T-t}(Y_{t})\Vert_{\text{HS}}^2\dif t}_{\text{Quadratic Variation Term}}\\
    &\quad+\underbrace{C\sum_{k=0}^{N-1}\int_{t_k}^{t_{k+1}}\sqrt{\mathbb{E}(\Vert\nabla\log p_{T-t}(Y_t)\Vert+\Vert\tilde{s}_t^*(Y_t)\Vert)^4}\cdot\sqrt{P(\tau_k<t)}\dif t}_{\text{Tail Event Term}}.
\end{aligned}$$

Here, the first term comes from the analyses in this section, which has been already calculated. The second term comes from the quadratic variation of the reverse process, and the third term quantifies the contribution of tail events, where the reverse process goes outside the local coordinate chart. In the following sections, we will eventually control the second and third term.

\subsection{Handle the Probability of Tail Events}

\label{sec:tail-event}

In this section, we provide upper bounds for the tail events term in the discretization error. First, according to Lemma \ref{lem:local-g}, within the local coordinate chart at $Y_{t_k}$ with radius $\omega=\frac{c}{\sqrt{d}}$, the metric tensor satisfies $\Vert g(x)\Vert\le 2$. Therefore, by Lemma \ref{lem:score-moment4},
$$\begin{aligned}
    \sqrt{\mathbb{E}(\Vert\nabla\log p_{T-t}(Y_t)\Vert+\Vert\tilde{s}_t^*(Y_t)\Vert)^4}\le C\sqrt{\mathbb{E}\Vert\nabla\log p_{T-t}(Y_t)\Vert^4}+C\sqrt{\mathbb{E}\Vert\nabla\log p_{T-t_k}(Y_t)\Vert^4}\le\frac{Cd}{T-t_{k+1}}.
\end{aligned}$$
Therefore, we only need to control the probability that the process $Y_t$ goes outside the local coordinate chart. Define the distance function $r(x):=\rho(Y_{t_k},x)$, which is smooth inside the local coordinate chart. Moreover, by Laplacian comparison theorem (see, e.g., Theorem 3.4.2 of \citealt{Hsu2002StochasticAO}), when the manifold $\mathcal{M}$ has nonnegative Ricci curvature, $\Delta r^2(x)\le 2d$. Then, using It\^o's rule, for any integer $m\ge 2$,
$$\begin{aligned}
    r(Y_{t\land\tau_k})^{2m}&=\int_{t_k}^{t\land\tau_k}\left\{\nabla\left[r(Y_u)^{2m}\right]\cdot(\nabla\log p_{T-u}(Y_u)\dif u+\dif B_u)+\frac{1}{2}\Delta\left[r(Y_u)^{2m}\right]\dif u\right\}\\
    &\le \int_{t_k}^{t\land\tau_k}\left(2m\cdot r(Y_u)^{2m-1}\Vert\nabla\log p_{T-u}(Y_u)\Vert+\frac{1}{2}\mathrm{div}(mr(Y_u)^{2m-2}\nabla r(Y_u)^2)\right)\dif u\\
    &\quad+2m\int_{t_k}^{t\land\tau_k}r(Y_u)^{2m-1}\cdot\dif \beta_u\\
    &\le\int_{t_k}^{t\land\tau_k}\left(2m\cdot r(Y_u)^{2m-1}\Vert\nabla\log p_{T-u}(Y_u)\Vert+(md+(2m(m-1))r(Y_u)^{2m-2}\right)\dif u\\
    &\quad+2m\int_{t_k}^{t\land\tau_k}r(Y_u)^{2m-1}\cdot\dif \beta_u,\\
\end{aligned}$$
where $(\beta_t)$ is a one-dimensional Brownian motion. Taking expectation, 
$$\begin{aligned}
\mathbb{E}\left[r(Y_{t\land\tau_k})^{2m}\right]
&\le\int_{t_k}^{t}\mathbb{E}\left[r(Y_{u\land\tau_k})^{2m-1}\cdot C_m\Vert\nabla\log p_{T-u}(Y_{u})\Vert\right]\dif u+\int_{t_k}^{t}\mathbb{E}\left[r(Y_{u\land\tau_k})^{2m-2}\cdot C_md\Vert\right]\dif u\\
&\overset{\text{(a)}}{\le}\int_{t_k}^{t}\bigg(\frac{\mathbb{E}[(r(Y_{u\land\tau_k})^{2m-1}/h_k^{(2m-1)/2m})^{2m/(2m-1)}]}{2m/(2m-1)}+\frac{\mathbb{E}[(C_m\Vert\nabla\log p_{T-u}(Y_u)\Vert h_k^{(2m-1)/2m})^{2m}]}{2m}\bigg)\dif u\\
&\quad+\int_{t_k}^{t}\bigg(\frac{\mathbb{E}[(r(Y_{u\land\tau_k})^{2m-2}/h_k^{(m-1)/m})^{m/(m-1)}]}{m/(m-1)}+\frac{\mathbb{E}[(C_md h_k^{(m-1)/m})^{m}]}{m}\bigg)\dif u\\
&\le\int_{t_k}^t\frac{2}{h_k}\mathbb{E}[r(Y_{u\land\tau_k})^{2m}]\dif u+C_mh_k^{2m-1}\int_{t_k}^t\mathbb{E}\Vert\nabla\log p_{T-u}(Y_u)\Vert^{2m}\dif u+C_mh_k^md^m,
\end{aligned}$$
where step (a) uses Young's inequality $xy\le\frac{x^p}{p}+\frac{x^q}{q}$ for $(p,q)$ satisfying $p>1,q>1$ and $p^{-1}+q^{-1}=1$. Since $\int_{t_k}^t\frac{2}{h_k}\dif u\le 2$, by Gr\"onwall's inequality,
$$\mathbb{E}\left[r(Y_{t\land\tau_k})^{2m}\right]\le C_mh_k^md^m+C_mh_k^{2m-1}\int_{t_k}^t\mathbb{E}\Vert\nabla\log p_{T-u}(Y_u)\Vert^{2m}\dif u.$$
Take $m=6$, then $$\mathbb{E}\left[r(Y_{t\land\tau_k})^{12}\right]\le Ch_k^6d^6+Ch_k^{12}d^6/(T-t_{k+1})^6\le Ch_k^6d^6,$$
if we assume that $h_k=O((T-t_{k+1})^{-1})$. Hence, by Markov's inequality,
$$\begin{aligned}
    P(\tau_k<t)=P(r(Y_{t\land\tau_k})\ge\omega/2)\le\frac{\mathbb{E}[r(Y_{t\land\tau_k})^{12}]}{(\omega/2)^{12}}\le Ch_k^6d^{12}.
\end{aligned}$$
Therefore, we conclude that
$$\text{Tail Event Term}\le C\sum_{k=0}^{N-1}h_k\cdot\frac{d}{T-t_{k+1}}\cdot h_k^3d^6\le C\sum_{k=0}^{N-1}\frac{h_k^4d^7}{T-t_{k+1}}.$$

\subsection{Sharp Analysis of Quadratic Variation Term}

\label{appendix:sharp}

In this section, we provide a sharp analysis for the following part of discretization error:
$$\text{Quadratic Variation Term}=C\sum_{k=0}^{N-1}h_k\int_{t_k}^{t_{k+1}}\mathbb{E}\Vert\nabla^2\log p_{T-t}(Y_{t})\Vert_{\text{HS}}^2\dif t.$$
As we will show in this section, this is the dominant term of discretization error, and it has a linear growth in $d$.

First, note that $\Vert\nabla\log p_{T-t}(x)\Vert^2$ is a smooth function in $x\in\mathcal{M}$. Hence, apply It\^o's lemma again on $\dif Y_t=\nabla\log p_{T-t}(Y_t)\dif t+\dif B_t$,
$$\begin{aligned}
    &\quad\dif\Vert\nabla\log p_{T-t}(Y_t)\Vert^2\\&=\partial_t\Vert\nabla\log p_{T-t}(Y_t)\Vert^2\dif t+\nabla\Vert\nabla\log p_{T-t}(Y_t)\Vert^2\cdot\nabla\log p_{T-t}(Y_t)\dif t\\
    &\quad+\frac{1}{2}\Delta\Vert\nabla\log p_{T-t}(Y_t)\Vert^2]\dif t+\nabla\Vert\nabla\log p_{T-t}(Y_t)\Vert^2\cdot\dif B_t\\
    &=2\nabla\log p_{T-t}(Y_t)\cdot\partial_t\nabla\log p_{T-t}(Y_t)\dif t\\
    &\quad+2\nabla\log p_{T-t}(Y_t)\cdot(\nabla^2\log p_{T-t}(Y_t)\cdot\nabla\log p_{T-t}(Y_t))\dif t\\
    &\quad+\nabla\log p_{T-t}(Y_t)\cdot\Delta\nabla\log p_{T-t}(Y_t)\dif t+\Vert\nabla\log p_{T-t}(Y_t)\Vert^2_{\text{HS}}\dif t\\
    &\quad+\nabla\Vert\nabla\log p_{T-t}(Y_t)\Vert^2\cdot\dif B_t\\
    &=\nabla\log p_{T-t}(Y_t)\cdot\bigg[-2\nabla^2\log p_{T-t}(Y_t)\cdot\nabla\log p_{T-t}(Y_t)-\nabla\Delta\log p_{T-t}(Y_t)\\
    &\quad+2\nabla^2\log p_{T-t}(Y_t)\cdot\nabla\log p_{T-t}(Y_t)+\Delta\nabla\log p_{T-t}(Y_t)\bigg]\dif t\\
    &\quad+\Vert\nabla\log p_{T-t}(Y_t)\Vert^2_{\text{HS}}\dif t+\nabla\Vert\nabla\log p_{T-t}(Y_t)\Vert^2\cdot\dif B_t\\
    &=\mathrm{Ric}(\nabla\log p_{T-t}(Y_t),\nabla\log p_{T-t}(Y_t))\dif t+\Vert\nabla\log p_{T-t}(Y_t)\Vert^2_{\text{HS}}\dif t\\
    &\quad+\nabla\Vert\nabla\log p_{T-t}(Y_t)\Vert^2\cdot\dif B_t.
\end{aligned}$$
Take expectation,
$$\begin{aligned}
    \frac{\dif}{\dif t}\mathbb{E}\Vert\nabla\log p_{T-t}(Y_t)\Vert^2&=\mathbb{E}\mathrm{Ric}(\nabla\log p_{T-t}(Y_t),\nabla\log p_{T-t}(Y_t))+\mathbb{E}\Vert\nabla\log p_{T-t}(Y_t)\Vert^2_{\text{HS}}\\
    &\ge\mathbb{E}\Vert\nabla\log p_{T-t}(Y_t)\Vert^2_{\text{HS}},
\end{aligned}$$
where the last inequality uses the assumption that $\mathcal{M}$ has nonnegative Ricci curvature. By product rule, we can immediately obtain that
$$\begin{aligned}
    \frac{\dif}{\dif t}\left((T-t)\mathbb{E}\Vert\nabla\log p_{T-t}(Y_t)\Vert^2\right)&=(T-t)\frac{\dif}{\dif t}\mathbb{E}\Vert\nabla\log p_{T-t}(Y_t)\Vert^2-\mathbb{E}\Vert\nabla\log p_{T-t}(Y_t)\Vert^2\\
    &\ge(T-t)\mathbb{E}\Vert\nabla\log p_{T-t}(Y_t)\Vert^2_{\text{HS}}-\mathbb{E}\Vert\nabla\log p_{T-t}(Y_t)\Vert^2.
\end{aligned}$$

Now, assume that the step sizes are chosen so that $h_k=O(h\cdot(T-t_{k+1}))$, then, let $\delta=T-t_N$,
$$\begin{aligned}
    &\quad\text{Quadratic Variation Term}\\&=C\sum_{k=0}^{N-1}h_k\int_{t_k}^{t_{k+1}}\mathbb{E}\Vert\nabla^2\log p_{T-t}(Y_{t})\Vert_{\text{HS}}^2\dif t.\\
    &\le C\sum_{k=0}^{N-1}h\int_{t_k}^{t_{k+1}}(T-t)\mathbb{E}\Vert\nabla^2\log p_{T-t}(Y_{t})\Vert_{\text{HS}}^2\dif t\\
    &\le Ch\int_{0}^{t_N}\left[\frac{\dif}{\dif t}\left((T-t)\mathbb{E}\Vert\nabla\log p_{T-t}(Y_t)\Vert^2\right)+\mathbb{E}\Vert\nabla\log p_{T-t}(Y_t)\Vert^2\right]\dif t\\
    &\le Ch\left[\delta\mathbb{E}\Vert\nabla\log p_{\delta}(Y_{T-\delta})\Vert^2-T\mathbb{E}\Vert\nabla\log p_{T}(Y_0)\Vert^2+\int_{\delta}^{T}\mathbb{E}\Vert\nabla\log p_{t}(X_t)\Vert^2\dif t\right]\\
    &\le Ch\left[\delta\mathbb{E}\Vert\nabla\log p_{\delta}(X_\delta)\Vert^2+\int_{\delta}^T\frac{d}{t}\dif t\right]\\
    &\le Ch\left[d+d\log(T/\delta)\right]\le Chd\log(T/\delta),
\end{aligned}$$
where we have used Lemma \ref{lem:score-bound} to control the second moment of the score function.

\subsection{Combining the Results}

We define the step size by $h_k=\min\{(T-\delta)-t_k,h\cdot T-t_k\}$, where $h>0$ is a parameter controlling the scale of step sizes. Then, as long as $h\le\frac{1}{2}$, the step size satisfies $h_k\le 2(T-t_{k+1})h$. Moreover, the total number of steps is given by $$N=\lceil\log_{1-h}(\delta/T)\rceil=\left\lceil\frac{\log(T/\delta)}{-\log(1-h)}\right\rceil=\Theta\left(h^{-1}\log(T/\delta)\right).$$

Now, with this schedule of step sizes, 
$$\begin{aligned}
    \text{Discretization Error}&\le Ch^2\sum_{k=0}^{N-1}h_k\left(d^4(T-t_{k+1})+d^3\right)\\
    &\quad+Ch^3\sum_{k=0}^{N-1}h_kd^7(T-t_{k+1})^2\\
    &\quad+Chd\log(T/\delta)\\
    &\le Ch^2\int_\delta^{T}d^4t\dif t+Ch^3\int_\delta^Td^7t^2\dif t+Chd\log(T/\delta)\\
    &\le Chd\log(T/\delta)+Ch^2d^4T^2+Ch^3d^7T^3.
\end{aligned}$$
This proves the error bound claimed in the theorem.

In addition, in order to achieve $\epsilon^2$ discretization error, it suffices to take $$h^{-1}=\Theta\left(\frac{d\log(T/\delta)}{\epsilon^2}+\frac{d^2T}{\epsilon}+\frac{d^{7/3}T}{\epsilon^{2/3}}\right),$$
where the corresponding number of steps is
$$N=\Theta\left(\left(\frac{d\log(T/\delta)}{\epsilon^2}+\frac{d^2T}{\epsilon}+\frac{d^{7/3}T}{\epsilon^{2/3}}\right)\log(T/\delta)\right).$$
When the target $\epsilon^2$ is sufficiently small, the formula can be simplified to $$N=\Theta\left(\frac{d\log^2(T/\delta)}{\epsilon^2}\right).$$
\section{Introduction to the Error Bound without Nonnegative Ricc Curvature Assumption}

\label{appendix:ricci}

In this section, we briefly introduce how to obtain the discretization error bound when the last assumption in Assumption \ref{assumption:manifold} is replaced by a bounded Ricci curvature assumption, where the manifold $\mathcal{M}$ has Ricci curvature bounded below by $-C_R$ for some constant $C_R>0$. A worst-case estimate gives $C_R=C_{\text{Rm}}d$, but we use $C_R$ here for generality.

Under this assumption, all the error partition in Appendix \ref{appendix:sde} holds, and the only difference is that the moment of true score function can be larger (see e.g., Lemma \ref{lem:score-bound-CR}). For simplicity, we only focus on the dominating term
$$\text{Quadratic Variation Term}=C\sum_{k=0}^{N-1}h_k\int_{t_k}^{t_{k+1}}\mathbb{E}\Vert\nabla^2\log p_{T-t}(Y_{t})\Vert_{\text{HS}}^2\dif t,$$
since the summation of other discretization error term is of order $o(h)$ for step size scale $h$.

To analyze the Hilbert-Schmidt term, we further apply It\^o's formula on $\Vert\nabla\log p_{T-t}(Y_t)\Vert^2$ as in Appendix \ref{appendix:sharp}, but with the new assumption,
$$\begin{aligned}
    \frac{\dif}{\dif t}\mathbb{E}\Vert\nabla\log p_{T-t}(Y_t)\Vert^2&=\mathbb{E}\mathrm{Ric}(\nabla\log p_{T-t}(Y_t),\nabla\log p_{T-t}(Y_t))+\mathbb{E}\Vert\nabla\log p_{T-t}(Y_t)\Vert^2_{\text{HS}}\\
    &\ge\mathbb{E}\Vert\nabla\log p_{T-t}(Y_t)\Vert^2_{\text{HS}}-C_R\mathbb{E}\Vert\nabla\log p_{T-t}(Y_t)\Vert^2,
\end{aligned}$$
Moreover, 
$$\begin{aligned}
    \frac{\dif}{\dif t}\left((T-t)\mathbb{E}\Vert\nabla\log p_{T-t}(Y_t)\Vert^2\right)&\ge(T-t)\mathbb{E}\Vert\nabla\log p_{T-t}(Y_t)\Vert^2_{\text{HS}}\\
    &\quad-(1+C_R(T-t))\mathbb{E}\Vert\nabla\log p_{T-t}(Y_t)\Vert^2,
\end{aligned}$$
Now, instead, we define the step sizes so that $h_k:=\min\{T-t_k,1\}\cdot h$, then, using Lemma \ref{lem:score-bound-CR},
$$\begin{aligned}
    &\quad\text{Quadratic Variation Term}\\&=C\sum_{k=0}^{N-1}h_k\int_{t_k}^{t_{k+1}}\mathbb{E}\Vert\nabla^2\log p_{T-t}(Y_{t})\Vert_{\text{HS}}^2\dif t.\\
    &\le Ch\left(\int_{0}^{T-\delta}\min\{T-t,1\}\mathbb{E}\Vert\nabla^2\log p_{T-t}(Y_{t})\Vert_{\text{HS}}^2\dif t\right)\\
    &\le Ch\bigg(-\int_{\delta}^{T}\frac{\dif}{\dif t}\left(\min\{t,1\}\mathbb{E}\Vert\nabla\log p_{t}(X_t)\Vert^2\right)\dif t\\
    &\quad+C_R\int_{\delta}^T\min\{t,1\}\mathbb{E}\Vert\nabla\log p_{t}(X_t)\Vert^2\dif t+\int_{\delta}^1\mathbb{E}\Vert\nabla\log p_{t}(X_t)\Vert^2\dif t\bigg)\\
    &\le Ch(d+\sqrt{dC_R}D)\left(\delta\left(\frac{1}{\delta}+C_R\right)+C_R\int_\delta^T\left(1+C_R\right)\dif t+\int_\delta^1\left(\frac{1}{t}+C_R\right)\dif t\right)\\
    &\le Ch(d+\sqrt{dC_R}D)\left(TC_R+\log(1/\delta)\right).
\end{aligned}$$

Note that when we use this step size schedule, the number of step sizes $N$ satisfies $N=O\left(h^{-1}(T+\log(1/\delta))\right)$. Hence, the number of steps required to achieve $\epsilon^2$-KL error satisfies $$N=O\left(\frac{(d+\sqrt{dC_R}D)(T^2C_R+TC_R\log(1/\delta)+\log^2(1/\delta))}{\epsilon^2}\right)=\tilde{O}\left(\frac{C_RT^2(d+\sqrt{dC_R}D)}{\epsilon^2}\right).$$
In addition, when we use the worst-case estimate $C_R=C_{\text{Rm}}d$, this leads to a bound given by $$N=\tilde{O}\left(\frac{T^2d^2}{\epsilon^2}\right).$$
Although it is still significantly better than the $\tilde{O}\left(\frac{T^2d^6}{\epsilon^2}\right)$ bound provided by \citet{xu2026polynomial}, it is worse than the $O\left(\frac{d\log(T/\delta)}{\epsilon^2}\right)$ bound we can obtain under the assumption $C_R=0$. This further highlights that Riemannian diffusion models do perform better on Riemannian manifolds with nonnegative Ricci curvature.

\section{Convergence Rate of Geodesic Random Walk}

\label{appendix:bm}

In the previous section, we have controlled the KL divergence between an estimated process $(\hat{Y}_t)_{0\le t\le T-\delta}$ and the true reverse process $(Y_t)_{0\le t\le T-\delta}$. The estimated process is defined by $$\dif\hat{Y}_t=\tilde{s}_t(\hat{Y}_t)\dif t+\dif B_t,$$
where the drift function $\tilde{s}_k(x):=\tilde{s}_t(x)$ is fixed within each interval $t\in[t_k,t_{k+1})$ given the state $\hat{Y}_{t_k}$ at time $t_k$.

Now, we use Geodesic Random Walk to approximate the estimated process. Specifically, for each interval $[t_k,t_{k+1})$, we partition the interval into $n_{\mathrm{bm}}$ small steps $(t_{k,0}=t_k,t_{k,1},t_{k,2},\dots,t_{k,n_k}=t_{k+1})$.
Given $\tilde{Y}_{t_k}$ at time $t_k$, we sample the state at time $t_{k+1}$ by the following iteration:
$$\tilde{Y}_{t_{k,i+1}}=\exp_{\tilde{Y}_{t_k,i}}\left((t_{k,i+1}-t_{k,i})\tilde{s}_{k}(\tilde{Y}_{t_{k,i}})+\sqrt{t_{k,i+1}-t_{k,i}}w_{k,i}\right),\quad i=0,\dots,n_k-1,$$
where $w_{k,i}$ is an independent Gaussian random vector in the tangent space $T_{\tilde{Y}_{t_k,i}}\mathcal{M}$. 

Through the Geodesic Random walk, we define a sampling process $\left(\tilde{Y}_{t_k'}\right)_{k=0}^{N'}$ where we use $t'$ to represent the discretization sequence given by $(t_{k,i})$ for notational simplicity. Define $\tilde{P}_i(y_{i+1}|(y_j)_{j\le i})$ as the transition kernel of process $\tilde{Y}$, that is, the conditional density of $\tilde{Y}_{t_{i+1}'}$ given $(\tilde{Y}_{t_{j}'})_{j\le i}$. Similarly, define $\hat{P}_i(y_{i+1}|(y_j)_{j\le i})$ as the corresponding transition kernel of $(\hat{Y}_{t_i'})_{i=0}^{N'}$. Then, 
$$\begin{aligned}
    \mathrm{KL}\left(\mathrm{Law}(\tilde{Y}_{T-\delta})\left\Vert\mathrm{Law}(\hat{Y}_{T-\delta})\right.\right)&\le\mathrm{KL}\left(\mathrm{Law}((\tilde{Y}_{t_i'})_{i=0}^{N'})\left\Vert\mathrm{Law}((\hat{Y}_{t_i'})_{i=0}^{N'})\right.\right)\\
    &=\sum_{i=0}^{N'}\mathbb{E}\left[\mathrm{KL}(\tilde{P}_i(\cdot|(\tilde{Y}_{t_j'})_{j\le i})\Vert\hat{P}_i(\cdot|(\tilde{Y}_{t_j'})_{j\le i}))\right].
\end{aligned}$$

This provides is general idea of controlling the Brownian-motion simulation error, and the details will be introduced below. In the following sections, we will apply heat kernel theories from Chapter 2 of \citep{Berline2004kernel} to write the explicit formula for the transition kernel $\hat{P}$. Though many high-level ideas comes from the analysis in \citep{xu2026polynomial}, their work has only argued that there is a polynomial guarantee on the convergence rate, and did not provide the explicit calculation on the degree of $d$ in the polynomial factor.

\subsection{Exact Formulation of Transition Kernel.}

Consider the interval $[t_{k},t_{k+1})$, and we have fixed the drift function $\tilde{s}_k(\cdot)$ in this interval. With the drift, the transition density $\hat{P}(t,x,y)$ satisfies the heat function $$(\partial_t+H)\hat{P}(t,x,y)=0,$$
where the generalized Laplacian $H$ is defined as $$H:=-\frac{1}{2}\Delta+\langle\tilde{s}_k(x),\nabla\rangle+\mathrm{div}(\tilde{s}_k(x)).$$

\subsubsection{Formal Solution of the Heat Equation.}

Assume that the initial point $y$ is fixed. Consider the local normal coordinate chart at $y$. For each point $x$ inside the chart, define $$j(x):=\sqrt{\det(g(x))},\qquad g\text{ is the Riemannian metric matrix.}$$
That is, $j(x)$ represents the volume element around $x$. Inside the normal coordinate chart, we define the Euclidean normal density as $$q_t(x,y):=\frac{1}{(2\pi t)^{d/2}}\exp\left(-\frac{\rho(x,y)^2}{2t}\right).$$
Moreover, we also define $$\hat{q}_t(x,y):=q_t(x,y)\cdot j(x)^{-1/2}.$$

With these definitions in mind, we now introduce the formal solution of the heat equation. Let $x_s=\exp_{y}(s\cdot\log_yx)$ be the geodesic curve from $y$ to $x$. Moreover, define the operator $B$ as $$B:=j^{1/2}\circ H\circ j^{-1/2},$$
where the scalar functions are viewed as multiplication operator, and $\circ$ means composition. Then we have the following theorem.

\begin{theorem}[Formal solution, Theorem 2.26 in \citealt{Berline2004kernel}]
There exists a unique formal solution $k_t(x,y)$ of the heat equation
$$(\partial_t+H)k_t(x,y)=0$$
of the form $$k_t(x,y)=\hat{q}_t(x,y)\sum_{i=0}^\infty t^iu_i(x,y),$$
such that $u_0(y,y)=1$. In addition, the explicit formula for $u_i$ is given by
$$\begin{aligned}
    u_0(x,y)&=\tau(x,y):=\exp\left(\int_0^1\langle\dot{x}_s,\tilde{s}_k(x_s)\rangle \dif s\right),\\
    u_i(x,y)&=-\tau(x,y)\int_0^1\tau(x_s,y)^{-1}s^{i-1}Bu_{i-1}(x_s,y)\dif s,\quad i\ge 1.
\end{aligned}$$
\label{thm:formal-solution}
\end{theorem}

Theorem \ref{thm:formal-solution} provide explicit formula for the formal solution. In fact, the definition of formal solution and Proposition 2.24 in \citep{Berline2004kernel} indicates that, if we define $$k^{(n)}_t(x,y):=\hat{q}_t(x,y)\sum_{i=0}^n t^iu_i(x,y)$$
as the clipped formal solution, this solution satisfies $$(\partial_t+H)k^{(n)}_t(x,y)=t^n\hat{q}_t(x,y)(Bu_n)(x,y).$$

\subsubsection{Parametrix Estimate}

According to Theorem \ref{thm:formal-solution}, we now focus on $k^{(1)}_t(x,y)$, which can be viewed as an approximation of the true heat kernel $\hat{P}$ of generalized Laplacian $H$. However, $k^{(1)}_t(x,y)$ is only defined locally. To extend $k^{(1)}_t(x,y)$ onto the whole manifold $\mathcal{M}$ and keep the smoothness, we further define 
$$\tilde{k}^{(1)}_t(x,y)=\eta_{\iota}(\rho(x,y))k^{(1)}_t(x,y),$$
where $\eta_{\iota}$ is the cut-off function with parameter $\iota>0$ defined later. This definition refers to \citep{Berline2004kernel}, which is also used in \citep{xu2026polynomial}.

By summarizing the conclusions in Theorem 2.19, Theorem 2.23 and Theorem 2.29 in \citep{Berline2004kernel}, we obtain the following exact representation of heat kernel. 

\begin{theorem}[Parametrix]
    Define the remainder $$r_t(x,y):=(\partial_t+H)\tilde{k}^{(1)}_t(x,y).$$
    Let $\Delta_n:=\{(t_1,\dots,t_n):0\le t_1\le t_2\le\cdots\le t_n\le 1\}$ be the standard simplex. Define 
    $$\begin{aligned}
        Q^{(n)}(t,x,y)&:=\int_{t\Delta_n}\int_{\mathcal{M}^n}\tilde{k}^{(1)}_{t-t_n}(x,x_n)r_{t_n-t_{n-1}}(x_n,x_{n-1})r_{t_{n-1}-t_{n-2}}(x_{n-1},x_{n-2})\\
        &\qquad\cdots r_{t_2-t_1}(x_2,x_1)r_{t_1}(x_1,y)\dif x_n\cdots\dif x_1\dif t_n\cdots\dif t_1,
    \end{aligned}$$
    and $Q^{(0)}(t,x,y):=\tilde{k}^{(1)}_t(x,y)$. Then, 
    $$\hat{P}(t,x,y)=\sum_{n=0}^\infty(-1)^nQ^{(n)}(t,x,y).$$
    \label{thm:parametrix}
\end{theorem}

By Theorem \ref{thm:parametrix}, $$\hat{P}(t,x,y)=\hat{q}_t(x,y)\tau(x,y)+t\hat{q}_t(x,y)u_1(x,y)+\sum_{n=1}^\infty(-1)^nQ^{(n)}(t,x,y).$$
Since the sampling process actually sample a Gaussian vector in the tangent space, intuitively, the transition kernel $\tilde{P}$ of the sampling process is close to $\hat{q}_t(x,y)\tau(x,y)$. Therefore, to control the difference between $\hat{P}$ and $\tilde{P}$, we need to provide bounds on
\begin{itemize}
    \item $t\hat{q}_t(x,y)u_1(x,y)$, which will be shown as of order $O(t)$, and
    \item $\sum_{n=1}^\infty(-1)^nQ^{(n)}(t,x,y)$, which will be shown as of order $o(t)$.
\end{itemize}

\subsection{Bounds on $u_1$ and its Derivatives}

In this section we provide bounds on $|u_1(x,y)|$ to control the first-order error term. In addition, we want to provide bounds on the derivatives to prepare for the later calculations.

We first expand the operator $B$ as follows:
$$\begin{aligned}
    B&=j^{1/2}\circ H\circ j^{-1/2}\\
    &=-\frac{1}{2}\Delta+\langle\tilde{s}_k(x)-j^{1/2}\nabla j^{-1/2},\nabla\rangle\\
    &\quad+\mathrm{div}(\tilde{s}_k(x))+\frac{1}{2}j^{1/2}\Delta j^{-1/2}+\langle V(x),j^{1/2}\nabla j^{-1/2}\rangle.
\end{aligned}$$
Moreover, in local coordinates, the Laplace-Beltrami operator can be written as $$\Delta=g^{\alpha\beta}\left[\partial_{\alpha\beta}-\Gamma_{\alpha\beta}^\gamma\partial_\gamma\right].$$

Therefore,
$$\begin{aligned}
    u_1(x,y)&=-\tau(x,y)\int_0^1\tau(x_s,y)^{-1}B_{x_s}(\tau(x_s,y))\dif s\\
    (\text{zeroth-order part})\;\;&=-\tau(x,y)\int_0^1\left(\text{div}(\tilde{s}_k(x_s))+\frac{1}{2}j^{1/2}\Delta j^{-1/2}+\langle \tilde{s}_k(x_s),j^{1/2}\nabla j^{-1/2}\rangle\right)\dif s\\
    (\text{first-order part})\;\;&\quad-\tau(x,y)\int_0^1\tau(x_s,y)^{-1}\cdot \dif s\\
    &\qquad\cdot\langle\tilde{s}_k(x_s)-j^{1/2}\nabla j^{-1/2}+g^{\alpha\beta}\Gamma_{\alpha\beta}^\gamma\partial_\gamma,\nabla_{x_s}\tau(x_s,y)\rangle\\
    (\text{second-order part})\;\;&\quad-\tau(x,y)\int_0^1\tau(x_s,y)^{-1}\cdot \dif s\cdot g^{\alpha\beta}\partial_{\alpha\beta}\tau(x_s,y)).
\end{aligned}$$
We first consider the zeroth-order part. By applying Lemma \ref{lem:local-g} to the local coordinate at $\hat{Y}_{t_k}$, $$|\mathrm{div}(\tilde{s}_k(x_s))|=|\partial_\alpha\tilde{s}_k^\alpha+\frac{1}{2}g^{\alpha\beta}\partial_{\gamma}g_{\alpha\beta}\tilde{s}_k^\gamma|\le C\Vert \hat{s}_{t_k}(\hat{Y}_{t_k})\Vert(\omega^{-1}1_{\omega/2\le\rho(x_s,\hat{Y}_{t_k})\le\omega}+d),$$
Moreover, by applying Lemma \ref{lem:local-g} to the local coordinate chart at $y$,
$$\partial_{\alpha}j=(\partial_{\alpha}g_{\beta\gamma})g^{\beta\gamma}j=j\cdot O(d\rho(x,y)),$$ 
$$\begin{aligned}
    \partial_{\alpha\beta}j&=(\partial_{\alpha}g_{\gamma\xi})g^{\gamma\xi}(\partial_{\beta}g_{\mu\zeta})g^{\mu\zeta}j+(\partial_{\alpha\beta}g_{\gamma\xi})g^{\gamma\xi}j+(\partial_\alpha g_{\gamma\xi})(\partial_\beta g^{\gamma\xi})j\\
    &=j\cdot O((d\rho(x,y))^2+d+d^2\rho(x,y)^2)=j\cdot O(d).
\end{aligned}$$
Thus, as long as $\rho(x,y)\le O(d^{-1/2})$, the zeroth-order part does not exceed $$\tau(x,y)\cdot O(\Vert \hat{s}_{t_k}(\hat{Y}_{t_k})\Vert d+d^2).$$

We now turn to the first-order part. Let $\log_yx=x^\alpha\partial_\alpha$ be the local normal coordinate representation of $x$. Then, 
$$\begin{aligned}
    \partial_\alpha\tau(x,y)
    &=\tau(x,y)\partial_{\alpha}\int_0^1\langle\dot{x}_s,\tilde{s}_k(x_s)\rangle\dif s\\
    &=\tau(x,y)\int_0^1\left[g_{\alpha\beta}\cdot\tilde{s}_k(x_s)^\beta+sx^\xi\left(g_{\xi\beta}\cdot\partial_\alpha\tilde{s}_k(x_s)^\beta+\partial_{\alpha}g_{\xi\beta}\cdot\tilde{s}_k(x_s)^\beta\right)\right]\dif s,
\end{aligned}$$
So, as long as $\rho(x,y)\le\iota=O(d^{-1/2})$, apply Lemma \ref{lem:local-g} on $g$ and sum over the elements,
$$\begin{aligned}
    \Vert\nabla\tau(x,y)\Vert&\le\tau(x,y)\cdot O\left(\Vert \hat{s}_{t_k}(\hat{Y}_{t_k})\Vert(1+\rho(x,y)(\omega^{-1}+d\omega+\rho(x,y)))\right)\\
    &\le\tau(x,y)\cdot O(\Vert\hat{s}_{t_k}(\hat{Y}_{t_k})\Vert).
\end{aligned}$$
This implies that the first-order part does not exceed $$\tau(x,y)\cdot O\left(\Vert\hat{s}_{t_k}(\hat{Y}_{t_k})\Vert\cdot\left(\Vert\hat{s}_{t_k}(\hat{Y}_{t_k})\Vert+d\rho(x,y)\right)\right).$$

Now, for the second-order part, we can also calculate it in the local normal coordinates: 
$$\begin{aligned}
    g^{\alpha\beta}\partial_{\alpha\beta}\tau(x,y)
    &=g^{\alpha\beta}\tau(x,y)\left(\partial_\alpha\int_0^1\langle\dot{x}_s,\tilde{s}_k(x_s)\rangle\dif s\right)\left(\partial_\beta\int_0^1\langle\dot{x}_s,\tilde{s}_k(x_s)\rangle\dif s\right)\\
    &\quad+g^{\alpha\beta}\tau(x,y)\int_0^1\bigg[(\partial_\beta g_{\alpha\gamma}+\partial_\alpha g_{\beta\gamma})\tilde{s}_k(x_s)^\gamma+s(g_{\alpha\gamma}\partial_\beta+g_{\beta\gamma}\partial_{\alpha})\tilde{s}_k(x_s)^\gamma\\
    &\qquad+s^2x^\xi(\partial_{\beta}g_{\xi\gamma}\cdot\partial_\alpha+\partial_{\alpha}g_{\xi\gamma}\cdot\partial_\beta+\partial_{\alpha\beta}g_{\xi\gamma}+g_{\xi\gamma}\partial_{\alpha\beta})\tilde{s}_k(x_s)^\gamma\bigg]\dif s\\
    &=\tau(x,y)\cdot O\bigg(\Vert\hat{s}_{t_k}(\hat{Y}_{t_k})\Vert^2+\Vert\hat{s}_{t_k}(\hat{Y}_{t_k})\Vert\left(d\rho(x,y)+\omega^{-1}+d\omega\right)\\
    &\quad+d\Vert\hat{s}_{t_k}(\hat{Y}_{t_k})\Vert\rho(x,y)\left((\omega^{-1}+d\omega)\rho(x,y)d+d+\omega^{-2}+(d\omega)^2\right)\bigg)\\
    &=\tau(x,y)\cdot O\left(\Vert\hat{s}_{t_k}(\hat{Y}_{t_k})\Vert^2+\Vert\hat{s}_{t_k}(\hat{Y}_{t_k})\Vert d^2\rho(x,y)\right).
\end{aligned}$$

\paragraph{Bound on $|u_1(x,y)|$.} Combining the above analysis, we obtain that, for $\rho(x,y)\le O(d^{-1/2})$, we have $$| u_1(x,y)|\le\tau(x,y)\cdot O\left(d^2+\Vert\hat{s}_{t_k}(\hat{Y}_{t_k})\Vert^2+\Vert\hat{s}_{t_k}(Y_{t_k})\Vert d^{1.5}\right).$$

\paragraph{Bound on $\Vert\nabla_x u_1(x,y)\Vert$.} Similarly as before, Lemma \ref{lem:local-g} shows that, for $\rho(x,y)\le O(d^{-1/2})$, $$\sum_{\alpha=1}^d\sum_{\beta=1}^dg^{\alpha\beta}=O(d),\quad\partial_{\gamma}g_{\alpha\beta}=O(d^{-1/2}),\quad \partial_{\gamma\eta}g_{\alpha\beta}=O(1).$$
Therefore, when take an additional derivative, when taking derivative on the term of form $g^{\alpha\beta}$, the sum of these elements is bounded by $\sum_{\alpha=1}^d\sum_{\beta=1}^d\partial_\gamma g^{\alpha\beta}=O(d^{3/2})$, and when taking derivative on the term of form $\partial_{\gamma}g_{\alpha\beta}$, the bound for each entry is  also multiplied by $O(d^{1/2})$. Moreover, by Lemma \ref{lem:local-g2}, differentiating terms of form $\partial_{\eta\gamma}g_{\alpha\beta}$ only multiplies the bound by a constant, when we assume Assumption \ref{assumption:manifold2}. Moreover, when differentiating the term $\tau(x,y)$ ouside the integral, it introduces an additional $\frac{\partial\tau}{\tau}=O\left(\Vert\hat{s}_{t_k}(\hat{Y}_{t_k})\Vert\right)$ term to the bound. We skip the detailed calculation here. By applying the product rule on all the terms, we eventually obtain the following bound:
$$\Vert\nabla u_1(x,y)\Vert\le\tau(x,y)\cdot O\left[\left(d^2+\Vert\hat{s}_{t_k}(\hat{Y}_{t_k})\Vert^2+\Vert\hat{s}_{t_k}(Y_{t_k})\Vert d^{1.5}\right)\cdot\left(\Vert\hat{s}_{t_k}(\hat{Y}_{t_k})\Vert+d^{0.5}\right)\right].$$

\paragraph{Bound on $\left|\Delta u_1(x,y)\right|$.} This part is almost the the same as calculating the gradient. However, taking the Laplacian sums over the partial derivatives, where the partial derivative $\partial_{\alpha\beta}$ is considered with weight $g^{\alpha\beta}$. Since $\sum_{\alpha=1}^d\sum_{\beta=1}^dg^{\alpha\beta}=O(d)$, this would introduce an additional $O(d)$ term into the bound. We also skip the detailed calculation for clarity, and eventually obtain the following bound:
$$|\Delta u_1(x,y)|\le\tau(x,y)\cdot O\left[d\left(d^2+\Vert\hat{s}_{t_k}(\hat{Y}_{t_k})\Vert^2+\Vert\hat{s}_{t_k}(Y_{t_k})\Vert d^{1.5}\right)\left(\Vert\hat{s}_{t_k}(\hat{Y}_{t_k})\Vert+d^{0.5}\right)^2\right].$$

\paragraph{Summarize.} According to Assumption \ref{assumption:score-estimate2}, $\Vert\hat{s}_{t_k}\Vert\le\frac{CD}{T-t_k}+\frac{C\sqrt{d}}{\sqrt{T-t_k}}$. Define $\delta_k:=T-t_k$ for simplicity. Then by applying the assumed bound on the score function, we deduce that, as long as $\rho(x,y)\le\frac{c}{\sqrt{d}}$ for some universal constant $c$, we have 
$$\begin{aligned}
    |u_1(x,y)|&\le\tau(x,y)\cdot O\left(d^2\delta_k^{-1/2}+D^2\delta_k^{-2}+d\delta_k^{-1}\right),\\
    \Vert\nabla u_1(x,y)\Vert&\le\tau(x,y)\cdot O\left(d^{5/2}\delta_k^{-1}+d^{3/2}\delta_k^{-1.5}+D^2d^{1/2}\delta_k^{-5/2}\right),\\
    |\Delta u_1(x,y)|&\le\tau(x,y)\cdot O\left(d^4\delta_k^{-3/2}+d^3\delta_k^{-2}+D^2d^2\delta_k^{-3}\right).
\end{aligned}$$


\subsection{Bound on the Infinite Series}

We first provide upper bounds on the absolute value of
$$\begin{aligned}
    r_t(x,y)=(\partial_t+H)\tilde{k}^{(1)}_t(x,y)=(\partial_t+H)[\eta_{\iota}(\rho(x,y))\hat{q}_t(x,y)(u_0(x,y)+tu_1(x,y))].
\end{aligned}$$
For the remainder $r_t$, we partition them into two parts $$r_t(x,y)=\hat{r}_t(x,y)+\tilde{r}_t(x,y),$$ where the first part $\hat{r}_t(x,y)$ only contains derivatives on $k^{(1)}_t(x,y)$, and the second part $\tilde{r}_t(x,y)$ takes at least one derivatives on $\eta_{\iota}(\rho(x,y))$. 

For the first part, by that $$(\partial_t+H)k^{(1)}_t(x,y)=t\hat{q}_t(x,y)[Bu_1(x,y)],$$
Hence, $$|\hat{r}_t(x,y)|\le 1_{\rho(x,y)\le\iota}t\hat{q}_t(x,y)\left|Bu_1(x,y)\right|.$$
Recall that 
$$\begin{aligned}
    B&=-\frac{1}{2}\Delta+\langle\tilde{s}_k(x)-j^{1/2}\nabla j^{-1/2},\nabla\rangle\\
    &\quad+\mathrm{div}(\tilde{s}_k(x))+\frac{1}{2}j^{1/2}\Delta j^{-1/2}+\langle V(x),j^{1/2}\nabla j^{-1/2}\rangle.
\end{aligned}$$
Then, apply the bounds from previous section and use Lemma \ref{lem:local-g},
$$\begin{aligned}
    |B_x u_1(x,y)|&\lesssim|\Delta u_1|+(\Vert\hat{s}_{t_k}(\hat{Y}_{t_k})\Vert+d\rho(x,y))\Vert\nabla u_1\Vert+O(\Vert\hat{s}_{t_k}(\hat{Y}_{t_k})\Vert d+d^2)|u_1|\\
    &\lesssim\tau(x,y)\cdot\left(d^4\delta_k^{-3/2}+d^3\delta_k^{-2}+D^2d\delta_k^{-3}\right)
\end{aligned}$$

For the second part, use that $H:=-\frac{1}{2}\Delta+\langle\tilde{s}_k(x),\nabla\rangle+\mathrm{div}(\tilde{s}_k(x))$, 
$$\begin{aligned}
    |\tilde{r}_t(x,y)|&\le |k^1(x,y)|\cdot\left|\Delta\eta_{\iota}(\rho(x,y))+\langle\tilde{s}_k(x),\nabla\eta_{\iota}(\rho(x,y))\rangle\right|+|\langle\nabla\eta_{\iota}(\rho(x,y)),\nabla k^1(x,y)\rangle|\\
    &\overset{\text{(a)}}{\le} C\hat{q}_t(x,y)\tau(x,y)1_{\iota/2\le\rho(x,y)\le\iota}\bigg[\frac{1}{\iota^2}+\frac{1}{\iota}|\Delta\rho(x,y)|+\iota^{-1}(D\delta_k^{-1}+d^{1/2}\delta_k^{-1/2})\\
    &\qquad+\iota^{-1}\left(\frac{|\nabla\rho(x,y)^2|}{2t}+j^{1/2}\nabla j^{-1/2}+\tau(x,y)^{-1}(|\nabla u_0|+t|\nabla u_1|)\right)\bigg]\\
    &\overset{\text{(b)}}{\le}C\hat{q}_t(x,y)\tau(x,y)1_{\iota/2\le\rho(x,y)\le\iota}\left[\frac{d}{\iota^2}+\frac{D\delta_k^{-1}+d^{1/2}\delta_k^{-1/2}}{\iota}\right]\\
    &\overset{\text{(c)}}{\le}C\hat{q}_t(x,y)\tau(x,y)1_{\iota/2\le\rho(x,y)\le\iota}\cdot\frac{1}{\iota t},
\end{aligned}$$
where we apply Lemma \ref{lem:local-g} for (b) and assumed that $t\le\min\{\tau d^{-1},\sup_{x\in B_{\iota}(y)}\{|u_1(x,y)|\}^{-1}\}$ for (a), (b) and (c).

In addition, as long as $\rho(x,y)\le c\left(D\delta_k^{-1}+d^{1/2}\delta_k^{-1/2}\right)$, by definition we know that $$c\le\tau(x,y)\le C$$
for universal constants $c$ and $C$, and that $$\Vert g(x)-I\Vert\le\frac{1}{d}$$
holds when $\rho(x,y)\le\frac{c}{\sqrt{d}}$, so that $c\le j(x)\le C$ holds for universal constants. Therefore, Let $A=C\cdot \left(d^4\delta_k^{-3/2}+d^3\delta_k^{-2}+D^2d\delta_k^{-3}\right)$, we can simply write
$$\begin{aligned}
    |\hat{r}_t(x,y)|&\le q_t(x,y)1_{\rho(x,y)\le\iota}\cdot A\\
    |\tilde{r}_t(x,y)|&\le q_t(x,y)1_{\iota/2\le\rho(x,y)\le\iota}\cdot\frac{1}{\iota t}.
\end{aligned}$$

Then, we are ready to control the infinite series. Focus on the term $Q^{(n)}(x,y)$. Following the steps by Lemma 24 by \citet{xu2026polynomial}, by taking $z_0=y$, and $s_{0}=0$,
$$\begin{aligned}
    \left|Q^{(k)}(t,x,y)\right|&\le A^n\int_{t\Delta^{n}}\dif s_1\cdots\dif s_n\int_{\mathcal{M}^{n}}q_{s_i}(x,z_n)1_{\rho(z_{i-1},z_i)\le\iota}\\
    &\quad\cdot\prod_{i=1}^{n}\hat{q}_{s_i-s_{i-1}}(z_{i},z_{i-1})\bigg((s_{i}-s_{i-1})1_{\rho(z_{i-1},z_i)\le\iota}\\
    &\qquad+\frac{1}{(s_i-s_{i-1})\iota}1_{\iota/2\le\rho(z_{i-1},z_i)\le\iota}\tilde{P}\bigg) \dif z_1\cdots \dif z_n\\
    &\le A^n\int_{t\Delta^{n}}(t^nI_{\text{loc}}(s)+I_{\text{out}}(s))ds,
\end{aligned}$$
where $I_{\text{loc}}$ is the integration over $$\mathcal{R}=\{(z_1,\dots,z_n):\rho(z_i,y)\le2\iota;\rho(z_i,z_{i-1})\le\iota/2\},$$
and $I_{\text{out}}$ performs integration over $\mathcal{R}^c$.

Now, we directly use the results from \citet{xu2026polynomial}. First they have showed that as long as $\iota\le\frac{1}{Cd^3}$, $$I_{\text{loc}}(s)\le 8^{k+1}\cdot q_t(x,y),\quad I_{\text{out}}(s)\le32^k(ct)^{-d/2}\exp\left(-\frac{\iota^2}{16td}\right).$$
Then, As long as $\rho(x,y)\le \iota/4d$, and $t\le\frac{c\iota^2}{d\log(d/\iota^2)}$, we have $$I_{\text{out}}(s)\le C\cdot 32^kq_{t}(x,y)t.$$
This finally implies that $$\left|Q^{(k)}(t,x,y)\right|\lesssim\frac{t^{2k} A^k8^k+32^k A^kt^{k+1}}{k!}.$$
Therefore, when $t\le(80 A)^{-1}$, summing over $k$ finally implies $$\sum_{k=1}^\infty\left|Q^{(k)}(t,x,y)\right|\le t^2A=o(t).$$

\subsection{Put All Things Together}

First, since the previous density bound are only satisfied in the local coordinate chart, we want to localize both sampling process $\tilde{Y}$ and the estimated reverse process $\hat{Y}$. Specifically, for each sampling step $\tilde{P}_i$, define an auxiliary process so that, as long as the sampled vector $\tilde{s}_k(\tilde{Y}_{t'_i})+\sqrt{t_{i+1}'-t_{i}'}w_{i'}>\iota/(4d)$, re-sample a new vector, and repeat this process until the requirement is satisfied. Define its transition density as $\hat{Q}_{\text{loc}}(t,x,y)$. Moreover, for the estimated process $\hat{Y}_t$ whose transition density is $P(t,x,y)$, define $$\hat{P}_{\text{loc}}(t,x,y)=\frac{\hat{P}(t,x,y)}{\int_{M}1_{\rho(x,y)\le\iota/(4d)}\hat{P}(t,x,y)\dif x},$$
that is, we perform the similar re-sample process. Clearly, the TV distance between sampling process and $\hat{Q}$ and between $P$ and $\hat{P}$ is exactly the probability of re-sampling. Use the similar bound derived from BDG inequality as in Appendix \ref{sec:tail-event}, this probability is bounded by $O(d^3(h_i')^3/(\iota/d)^6)=O(d^9(h_i')^3\iota^{-6})=o((h_i')^2)$, where $h_i'$ represents the step size for Brownian motion simulation.

Moreover, we can directly write the transition density $\hat{Q}_{\text{loc}}(t,x,y)$ as
$$\hat{Q}_{\text{loc}}(t,x,y)=\frac{1_{\rho(x,y)\le\iota/4d}}{P(\text{resample})}Q(t,x,y),$$
where $$Q(t,x,y):=\frac{1}{(2\pi t)^{d/2}}\exp\left(-\frac{|\log_y x-\tilde{s}_k(y)\cdot t|^2}{2t}\right)j(x)^{-1}$$
and the dot product is done in Euclidean space.

\subsubsection{Analysing the Density Ratio}

Now, for $\rho(x,y)\le\iota/4d$, we provide an upper bound for
$$\begin{aligned}
    \left|\frac{\hat{P}(t,x,y)}{Q(t,x,y)}-1\right|
    &\le C\left|\frac{\hat{P}(t,x,y)}{\hat{q}_t(x,y)u_0(x,y)}-1\right|\\
    &+C\left|\frac{\hat{q}_t(x,y)}{q_t(x,y)j(x)^{-1}}-1\right|\\
    &+C\left|\frac{q_t(x,y)u_0(x,y)j(x)^{-1}}{Q(t,x,y)}-1\right|.
\end{aligned}$$

By the previous analyses, the first term is bounded by $C\left(d^2\delta_k^{-0.5}+D^2\delta_k^{-2}+d\delta_k^{-1}\right)t$ for small $t$.

For the second term, it equals to $j(x)^{-1/2}$, and then the difference is upper bounded by $$|j(x)^{-1/2}-1|\le(1+C\rho(x,y)^2)^{d/2}-1\le Cd\rho(x,y)^2,$$
since $\rho(x,y)\le\frac{c}{\sqrt{d}}$.

For the third term, directly expanding the term, we get $$\begin{aligned}
    &\quad\left|\frac{q_t(x,y)u_0(x,y)j(x)^{-1}}{Q(t,x,y)}-1\right|\\
    &=\left|\exp\left(\frac{|\log_yx-t\cdot\tilde{s}_k(y)|^2}{2t}-\frac{\rho(x,y)^2}{2t}-\int_0^1\langle\dot{x}_s,\tilde{s}_k(x_s)\rangle_g ds\right)-1\right|\\
    &=\left|\exp\left(t|\tilde{s}_k(y)|^2+\int_0^1 x^\alpha(\delta_{\alpha\beta}\tilde{s}_k(y)^\beta+g_{\alpha\beta}\tilde{s}_k(x_s)^\beta)\dif s\right)-1\right|.
\end{aligned}$$
Clearly, when $t<\left(\frac{D^2}{\delta_k^2}+\frac{d}{\delta_k}\right)^{-1}$,
$$\exp\left(t|\tilde{s}_k(y)|\right)=1+O(D^2\delta_k^{-2}+d\delta_k^{-1})\cdot t.$$
Moreover, $|g(x)-I|\le C\rho(x,y)^2$, and by the definition of $\tilde{s}_k$, $$|\nabla\tilde{s}_k(x)|\le\Vert \hat{s}_{t_k}(\hat{Y}_{t_k})\Vert(\omega^{-1}+\sqrt{d}\max\Gamma)\le\sqrt{d}\cdot O(D\delta_k^{-1}+d^{1/2}\delta_k^{-1/2}).$$
So,
$$\left|\int_0^1x^\alpha(\delta_{\alpha\beta}\tilde{s}_k(y)^\beta+g_{\alpha\beta}\tilde{s}_k(x_s)^\beta)\dif s\right|\le\rho(x,y)^3\Vert\tilde{s}_k(x_s)\Vert+\max|\nabla\tilde{s}_k(x)|\rho(x,y)^2.$$

Therefore, for $\rho(x,y)\le\iota/(4d)$, combining Assumption \ref{assumption:score-estimate2} and our previous analyses on $\nabla\tilde{s}_k$,
$$\begin{aligned}
    \left|\frac{\hat{P}(t,x,y)}{Q(t,x,y)}-1\right|&\lesssim (d\delta_k^{-1/2}+Dd^{1/2}\delta_k^{-1})\rho(x,y)^2+(d^{1/2}\delta_k^{-1/2}+D\delta_k^{-1})\rho(x,y)^3\\
    &+(D^2\delta_k^{-2}+d\delta_k^{-1}+d^2\delta_k^{-1/2})\cdot t.
\end{aligned}$$

\subsection{Controlling the KL divergence}

Now we compute $\text{KL}(\hat{Q}_{\text{loc}}\Vert\hat{P}_{\text{loc}})$. by that $\log(1+x)\ge x-Cx^2$ for $x\in[-1/2,1/2]$, 
$$\begin{aligned}
    \mathrm{KL}(\hat{Q}_{\text{loc}}\Vert\hat{P}_{\text{loc}})&=-\int_M\hat{Q}_{\text{loc}}(t,x,y)\log\frac{\hat{P}_{\text{loc}}(t,x,y)}{\hat{Q}_{\text{loc}}(t,x,y)}\dif x\\
    &\le-\int_{\mathcal{M}}\hat{Q}_{\text{loc}}(t,x,y)\left[\frac{\hat{P}_{\text{loc}}(t,x,y)}{\hat{Q}_{\text{loc}}(t,x,y)}-1-C\left|\frac{\hat{P}_{\text{loc}}(t,x,y)}{\hat{Q}_{\text{loc}}(t,x,y)}-1\right|^2\right]\dif x\\
    &\le C\int_{\mathcal{M}}\hat{Q}_{\text{loc}}(t,x,y)\left|\frac{\hat{P}_{\text{loc}}(t,x,y)}{\hat{Q}(t,x,y)}-1\right|^2\dif x,
\end{aligned}$$
where we have used that $$\int_{\mathcal{M}}\hat{Q}(t,x,y)\dif x=\int_{\mathcal{M}}\hat{P}(t,x,y)\dif x=1.$$
Therefore, when $t$ is sufficiently small, 
$$\begin{aligned}
    \mathrm{KL}(\hat{Q}_{\text{loc}}\Vert\hat{P}_{\text{loc}})
    &\le C\mathbb{E}_{\hat{Q}_{\text{loc}}}\left|\frac{\hat{P}(t,x,y)}{Q(t,x,y)}-1\right|^2+C\sup\left|\frac{\hat{Q}_{\text{loc}}}{Q}-1\right|^2+C\sup\left|\frac{\hat{P}_{\text{loc}}}{\hat{P}}-1\right|^2\\
    &\le C(d^2\delta_k^{-1}+D^2d\delta_k^{-2})\mathbb{E}_{\hat{Q}_{\text{loc}}}[\rho(x,y)^4]\\
    &\quad+C(d\delta_k^{-1}+D^2\delta_k^{-2})\mathbb{E}_{\hat{Q}_{\text{loc}}}[\rho(x,y)^6]\\
    &\quad+C(D^4\delta_k^{-4}+d^2\delta_k^{-2}+d^4\delta_k^{-1})t^2+o(t^4)\\
    &\le C\left(D^4\delta_k^{-4}+d^4\delta_k^{-1}+D^2\delta_k^{-2}d^3\right)t^2.
\end{aligned}$$

Therefore, let $\tilde{Y}_{\text{loc}}$ and $\hat{Y}_{\text{loc}}$ be two processes defined using localized transition density $\hat{Q}_{\text{loc}}$ and $\hat{P}_{\text{loc}}$, respectively. Then, given that the step sizes satisfy $h_i':=t_{i+1}'-t_i'\le h'\delta_k$ for a scale parameter $h'>0$, the KL divergence between the distributions of $\tilde{Y}_{\text{loc},T-\delta}$ and $\hat{Y}_{\text{loc},T-\delta}$ is bounded by
$$\begin{aligned}
    \mathrm{KL}\left(\mathrm{Law}(\tilde{Y}_{\text{loc},T-\delta})\left\Vert\mathrm{Law}(\hat{Y}_{\text{loc},T-\delta})\right.\right)&\le h'\sum_{i'=0}^{N'-1}h_i\left(\frac{D^4}{(T-t_i')^3}+d^4+\frac{D^2d^3}{T-t_i'}\right)\\
    &\le h'\left(d^4T+\delta^{-2}D^4+D^2d^3\log(T/\delta)\right).
\end{aligned}$$

\paragraph{Proof of Theorem \ref{thm:bm}}

With Pinsker's inequality and the triangle inequality for TV distance, 
$$\begin{aligned}
    \mathrm{TV}\left(\mathrm{Law}(\tilde{Y}_{T-\delta}),\mathrm{Law}(\hat{Y}_{T-\delta})\right)^2&\lesssim\sqrt{\mathrm{KL}\left(\mathrm{Law}(\tilde{Y}_{\text{loc},T-\delta})\left\Vert\mathrm{Law}(\hat{Y}_{\text{loc},T-\delta})\right.\right)}\\
    &\quad+\mathrm{TV}\left(\mathrm{Law}(\hat{Y}_{\text{loc},T-\delta}),\mathrm{Law}(\hat{Y}_{T-\delta})\right)^2\\
    &\quad+\mathrm{TV}\left(\mathrm{Law}(\tilde{Y}_{T-\delta}),\mathrm{Law}(\tilde{Y}_{\text{loc},T-\delta})\right)^2\\
    &\lesssim h'\left(d^4T+\delta^{-2}D^4+D^2d^3\log(T/\delta)\right).
\end{aligned}$$

\subsection{Proof of Corollary \ref{cor:num-of-steps}}

\label{appendix:proof-of-num-steps}

\begin{proof}[proof of Corollary \ref{cor:num-of-steps}]
    Using Pinsker's inequality, triangle inequality for TV distance, and the data-processing inequality,
    $$\mathrm{TV}(p_\delta,\tilde{q}_{T-\delta})\lesssim \mathrm{TV}(p_T,\hat{q}_{0})+\sqrt{\mathrm{KL}(p_\delta,\hat{q}_{T-\delta}^*)}+\mathrm{TV}(\hat{q}_{T-\delta},\tilde{q}_{T-\delta}),$$
    where $\hat{q}^*$ refer to the distribution defined by equation (\ref{eq:reverse-riemannian-estimate}), but replace the initial distribution $\hat{q}_0$ by $p_T$ instead.
    Then, apply Theorem \ref{thm:sde} on the KL divergence and apply Theorem \ref{thm:bm} on the last TV distance, as long as the step sizes $h$ and $h'$ are sufficiently small,
    $$\mathrm{TV}(p_\delta,\tilde{q}_{T-\delta})\lesssim\mathrm{TV}(p_T,\hat{q}_{0})+\epsilon_{\text{score}}+\sqrt{hd\log(T/\delta)}+\sqrt{h'd^4T}.$$
    Therefore, by choosing $h^{-1}=\Theta(\epsilon^{-2}d\log(T/\delta))$ and $(h')^{-1}=\Theta\left(\epsilon^{-2}d^4T\right)$, we can achieve $O(\epsilon)$ TV error. Finally, we conclude using the relations $N=\Theta\left(h^{-1}\log(T/\delta)\right)$ and $N'=\Theta\left((h')^{-1}\log(T/\delta)\right)$, which comes from our definition of step size schedule.
\end{proof}
\section{Useful Lemmas}

First, we provide some upper bounds on the score functions $\nabla\log p_t(x)$. Although the results can be found in many literatures (e.g., \citet{ENGOULATOV2006518,hsu1999derivatives}), the dimension $d$ are often dealt as a constant and hided in the formulas. Therefore, we first summarize a useful results in the following lemma, where the dependence on $d$ are calculated explicitly.

Let $H(t,x,y)$ be the transition kernel of standard Brownian motion on $\mathcal{M}$.

\begin{lemma}
    Suppose the Ricci curvature is bounded below by $-C_R$, where $C_R\ge0$. Then,
    $$\left\Vert\nabla\log H(t,x,y)\right\Vert\le C\sqrt{\left(\frac{1}{t}+C_R\right)\left(d+\frac{\rho(x,y)^2}{t}+\sqrt{C_Rd}\rho(x,y)\right)}.$$
    \label{lem:log-heat-kernel-derivative}
\end{lemma}

\begin{proof}
    This lemma rewrites the results by \citet{ENGOULATOV2006518}, where the only difference is that we write the $d$-related constants explicitly. First, corollary 6 by \citet{ENGOULATOV2006518} shows that
    $$\left\Vert\nabla\log H(t,x,y)\right\Vert^2\le 2\left(\frac{1}{t}+C_R\right)\mathbb{E}\log\frac{H(t/2,X_{t/2},y)}{H(t,x,y)}.$$
    Moreover, according to Equation (22) from \citet{ENGOULATOV2006518}, for all $s>0,t>0$ and $x,y\in\mathcal{M}$,
    $$\frac{H(t,x)}{H(t+s,y)}\le\left(\frac{t+s}{t}\right)^{n/2}\times\exp\left(\frac{(d(x,y)+\sqrt{C_Rd}s)^2}{4s}+\frac{\sqrt{C_Rd}}{4}\min\{\rho(x,y),\sqrt{C_Rd}s\}\right).$$
    Replace $t$ and $s$ by $t/2$, we obtain
    $$\begin{aligned}
        \mathbb{E}\log\frac{H(t/2,X_{t/2},y)}{H(t,x,y)}&\le\frac{n}{2}\log 2+\frac{(\rho(x,y)+\sqrt{C_Rd}t/2)^2}{2t}+\frac{\sqrt{C_Rd}}{4}\min\{\rho(x,y),\sqrt{C_Rd}t/2\}\\
        &=C_1\left(d+\frac{\rho(x,y)^2}{t}+\sqrt{C_Rd}\rho(x,y)\right),
    \end{aligned}$$
    where $C_1$ is a universal constant.

    Hence,
    $$\begin{aligned}
        \left\Vert\nabla\log H(t,x,y)\right\Vert&\le\sqrt{2\left(\frac{1}{t}+R\right)\mathbb{E}\log\frac{H(t/2,X_{t/2},y)}{H(t,x,y)}}\\
        &\le C\sqrt{\left(\frac{1}{t}+C_R\right)\left(d+\frac{\rho(x,y)^2}{t}+\sqrt{C_Rd}\rho(x,y)\right)}
    \end{aligned}$$
\end{proof}

The next lemma is a direct application of the former lemma, and is widely used in our analysis. First, recall the forward process is the standard brownian motion on manifold $\mathcal{M}$ with initial distribution $p_0$.

\begin{lemma}
    Under Assumption \ref{assumption:manifold}, for all $t>0$ and $x\in\mathcal{M}$, the score function $\nabla\log p_t(x)$ satisfies
    $$\Vert\nabla\log p_t(x)\Vert\le C\left(\frac{\sqrt{d}}{\sqrt{t}}+\frac{D}{t}\right).$$
    Moreover, for all $t>0$, $$\mathbb{E}_{x\sim p_t}\left[\Vert\nabla\log p_t(x)\Vert^2\right]\le\frac{Cd}{t}.$$
    \label{lem:score-bound}
\end{lemma}

\begin{proof}
    For $0\le s<t$, we can represent $$p_t(x)=\int_{\mathcal{M}}p_s(\dif y)H(t-s,x,y).$$
    Then, 
    $$\begin{aligned}
        \nabla\log p_t(x)=\frac{\nabla p_t(x)}{p_t(x)}&=\frac{\int_{\mathcal{M}}p_{0}(\dif y)\nabla_xH(t,x,y)}{\int_Mp_{0}(\dif y)H(t,x,y)}\\
        &=\frac{\int_{\mathcal{M}}p_{0}(\dif y)H(t,x,y)\nabla_x\log H(t,x,y)}{\int_{\mathcal{M}}p_{0}(\dif y)H(t,x,y)}\\
        &=\mathbb{E}\left[\nabla_x\log H(t,x,X_0)|X_t=x\right].
    \end{aligned}$$
    By Lemma \ref{lem:score-bound}, since Assumption \ref{assumption:manifold} ensures that the Ricci curvature is nonnegative, we can use $C_R=0$ in Lemma \ref{lem:log-heat-kernel-derivative} to deduce that $$\left\Vert\nabla\log H(t,x,y)\right\Vert\le C\left(\frac{\sqrt{d}}{\sqrt{t}}+\frac{\rho(x,y)}{t}\right).$$
    Hence, by that $\rho(x,y)\le D$, we directly obtain a uniform upper-bound, which is the first claim in Lemma \ref{lem:score-bound}:
    $$\Vert\nabla\log p_t(x)\Vert\le\sup_{y\in\mathcal{M}}\Vert\nabla_x\log H(t,x,y)\Vert\le C\left(\frac{\sqrt{d}}{\sqrt{t}}+\frac{D}{t}\right).$$
    In addition, taking expectation over $x\sim p_t$, we have
    $$\begin{aligned}
        \mathbb{E}_{x\sim p_t}\left[\Vert\nabla\log p_t(x)\Vert^2\right]&=\mathbb{E}_{x\sim p_t}\left[\left\Vert\mathbb{E}_{y\sim p_{0|t}}[\nabla_x\log H(t,x,y)]\right\Vert^2\right]\\
        &\overset{\text{(a)}}{\le}\mathbb{E}_{x\sim p_t}\mathbb{E}_{y\sim p_{0|t}}\left[\Vert\nabla_x\log H(t,x,y)\Vert^2\right]\\
        &\le C\cdot\mathbb{E}_{(y,x)\sim p_{0,t}}\left[\frac{d}{t}+\frac{\rho(x,y)^2}{t^2}\right]\\
        &\overset{\text{(b)}}{=}C\cdot\frac{d}{t}+\frac{C}{t^2}\cdot\mathbb{E}\left[\rho(X_t,X_0)^2\right].
    \end{aligned}$$
    Here, step (a) comes from Jensen's inequality applied on convex function $\Vert\cdot\Vert^2$, and step (b) is by noticing that $(y,x)\sim p_{0,t}$ have the same distribution as $X_0,X_t$, and recall that the forward process $(X_t)_{t\ge0}$ is a standard Brownian motion on $\mathcal{M}$.

    We now upper-bound the expectation of $\rho(X_t,X_0)^2$. By fixing $X_0$ and take $r(x)=\rho(x,X_0)$, the function $r(\cdot)$ is smooth on $\mathcal{M}\setminus(\text{Cut}(X_0)\cap\{X_0\})$, where $\text{Cut}(X_0)$ represents the cut locus of $X_0$ on $\mathcal{M}$. Moreover, the Laplacian comparison theorem (see, e.g., Theorem 3.4.2 of \citealt{Hsu2002StochasticAO}) shows that, when the Ricci curvature is nonnegative, $$\Delta r(x)\le\frac{d-1}{r(x)},\quad x\in\mathcal{M}\setminus(\text{Cut}(X_0)\cap\{X_0\}).$$
    With this in mind, by applying Ito's lemma on $r(X_t)$, we obtain that $$\begin{aligned}
        \dif r^2(X_t)&=\langle\nabla r^2(X_t),\dif X_t\rangle+\frac{1}{2}\Delta r^2(X_t)\dif t\\
        &=2r(X_t)\cdot \dif\beta_t+\frac{1}{2}(2r(X_t)\Delta r(X_t)+2\Vert\nabla r(X_t)\Vert^2)\dif t\\
        &\le2r(X_t)\cdot \dif\beta_t+\left(\frac{(d-1)r(X_t)}{r(X_t)}+1\right)\dif t\\
        &=2r(X_t)\cdot\dif\beta_t+d\cdot\dif t,
    \end{aligned}$$
    where $\beta_t$ is a 1-dimensional standard Brownian motion in $\mathbb{R}$.
    Note that $r^2(x)$ is smooth at $X_0$, so the formula is valid until $X_t$ hits the cut locus. However, we can apply the Calabi's trick on the cut locus to show that the formula holds globally. Specifically, let $x_0\in\text{Cut}(X_0)$, we pick a minimizing geodesic $\gamma:[0,r(x)]\to\mathcal{M}$ from $X_0$ to $x_0$, and then define $r_\epsilon(x)=\epsilon+\rho(\gamma(\epsilon),x)$. Clearly, $r_\epsilon(x_0)=r(x_0)$ and $r_\epsilon\ge r$ on $\mathcal{M}$. Moreover, Calabi's trick ensures that $r_\epsilon$ is smooth around $x_0$. In addition, $\Delta r_\epsilon^2(x)\le 2d$ also holds for all $\epsilon>0$. Now, suppose $X_s\in\text{Cut}(X_0)$, by applying Ito's formula on $r_\epsilon$, we obtain that locally, 
    $$r^2(X_t)-r^2(X_s)\le r_\epsilon^2(X_t)-r_\epsilon^2(X_s)=\int_s^t2\rho(\gamma_\epsilon,X_r)\dif\tilde{\beta}_r+d(t-s),$$
    Taking $\epsilon\to0$, and taking summation of all local hits on the cut locus, we show that $$r^2(X_t)\le\int_0^t 2r(X_s)d\beta_s+d\cdot t$$
    holds for all $t>0$. Therefore, by taking expectation, $$\mathbb{E}[r^2(X_t)]\le dt,$$
    which immediately implies that
    $$\begin{aligned}
        \mathbb{E}_{x\sim p_t}\left[\Vert\nabla\log p_t(x)\Vert^2\right]&\le C\cdot\frac{d}{t}+\frac{C}{t^2}\cdot\mathbb{E}\left[\rho(X_t,X_0)^2\right]\\
        &\le\frac{Cd}{t}.
    \end{aligned}$$
    This concludes the second claim of Lemma \ref{lem:score-bound}.
\end{proof}

The second inequality of Lemma \ref{lem:score-bound} provides an upper bound for the second moment of $\nabla\log p_t(x)$. The following lemma further computes a bound for the higher moments.
\begin{lemma}
    Under Assumption \ref{assumption:manifold}, for all $t>0$,
    $$\mathbb{E}_{x\sim p_t}\left[\Vert\nabla\log p_{t}(x)\Vert^4\right]\le\frac{Cd^2}{t^2}.$$
    More generally, for any $p>1$,
    $$\mathbb{E}_{x\sim p_t}\left[\Vert\nabla\log p_{t}(x)\Vert^{2p}\right]\le\frac{C_pd^p}{t^p},$$
    where the constant $C_p$ only depends on $p$.
    \label{lem:score-moment4}
\end{lemma}

\begin{proof}
    Similarly as the proof of Lemma \ref{lem:score-bound}, using Jensen's inequality,
    $$\mathbb{E}_{x\sim p_t}\left[\Vert\nabla\log p_{t}(x)\Vert^{2p}\right]\le\frac{C_pd^p}{t^p}+\frac{C_p}{t^{2p}}\cdot\mathbb{E}\left[\rho(X_t,X_0)^{2p}\right].$$
    Let $r(x):=\rho(x,X_0)$. We have proved that $$r^2(X_t)\le\int_0^tr(X_s)\dif\beta_s+d\cdot t.$$
    Using $(x^p)'=px^{p-1}$ and $(x^p)''=p(p-1)x^{p-2}$ apply It\^o's lemma on the right hand side, we have $$r^{2p}(X_t)\le\int_0^tpr^{2p-2}(X_s)\cdot\left[r(X_s)\dif \beta_s+d\cdot \dif s\right]+\frac{1}{2}\int_0^tp(p-1)r^{2(p-2)}(X_s)r(X_s)^2\dif s$$
    Taking expectation, we obtain that 
    $$\begin{aligned}
        \mathbb{E}[r^{2p}(X_t)]&\le\int_0^t\mathbb{E}[r^{2p-2}(X_s)]\cdot C_pd\dif s\\
        &\le\int_0^t\left(\mathbb{E}[r^{2p-2}(X_s)]/t^{(p-1)/p}\right)\cdot C_pdt^{(p-1)/p}\dif s\\
        (\text{Young's inequality})\quad&\le\int_0^t\left\{\frac{\mathbb{E}[r^{2p-2}(X_s)]^{p/(p-1)}}{tp/(p-1)}+\frac{C_pd^pt^{p-1}}{p}\right\}\dif s\\
        &\le\int_0^t\frac{p-1}{tp}\mathbb{E}[r^{2p}(X_s)]\dif s+C_pd^pt^p.
    \end{aligned}$$
    By Gr\"onwall's inequality, $$\mathbb{E}[r^{2p}(X_t)]\le\exp\left(\int_0^t\frac{p-1}{tp}\dif s\right)C_pd^pt^p\le C_pd^pt^p.$$
    This concludes the desired results.
\end{proof}

In addition, we provide a bound to compare with the second inequality in Lemma \ref{lem:score-bound}. In the next lemma, we do not assume that the Ricci curvature to be nonnegative, but instead, we assume that the Ricci curvature is bounded below by $-C_R$ for $C_R\ge0$. When $C_R=0$, it immediately reduces to the nonnegative Ricci curvature case.

\begin{lemma}
    Under Assumption \ref{assumption:manifold}, but instead we assume that the Ricci curvature is bounded below by $-C_R$. Then, $$\mathbb{E}_{x\sim p_t}\Vert\nabla\log p_{t}(x)\Vert^2\le C\left(\frac{1}{t}+C_R\right)\left(d+\sqrt{dC_R}D\right).$$
    \label{lem:score-bound-CR}
\end{lemma}

\begin{proof}
    Similarly as before, we still need to compute $\mathbb{E}[r(X_t)^2]$ for $r(x):=\rho(X_0,X_t)$. However, in this case, Laplacian comparison theorem indicates that $$\Delta r\le (d-1)k\coth(kr),$$
    where $k:=\sqrt{\frac{C_R}{d-1}}$.
    Thus, $$\frac{1}{2}\Delta r^2\le\Vert\nabla r\Vert^2+r\Delta r\le 1+(d-1)kr\coth(kr)\le 1+(d-1)(1+kr)=d+(d-1)kr,$$
    where we have used the fact that $x\coth x\le 1+x$ for $x\ge0$. Therefore, by applying the same analysis as before, we can obtain that
    $$\begin{aligned}
        \mathbb{E}[r^2(X_t)]&\le\int_0^t\mathbb{E}\left[d+(d-1)kr(X_s)\right]\dif s\le \left(d+\sqrt{(d-1)C_R}D\right)t
    \end{aligned}$$
    Finally, applying Lemma \ref{lem:log-heat-kernel-derivative}, and apply the same analysis with Jensen's inequality,
    $$\begin{aligned}
        \mathbb{E}_{x\sim p_t}\Vert\nabla\log p_t(x)\Vert^2\le C\left(\frac{1}{t}+C_R\right)\left(d+\sqrt{(d-1)C_R}D\right)
    \end{aligned}$$
\end{proof}

The next lemma by \citet{xu2026polynomial} is very useful, which quantifies the local behavior of Riemannian metric matrix.

\begin{lemma}[Lemma 5 by \citet{xu2026polynomial}]
    Let $y\in M$. Under Assumption \ref{assumption:manifold}, the Riemannian metric matrix $g(x)$ under the local normal coordinates at $y$ satisfies
    $$\begin{aligned}
        \Vert g(y)-I\Vert&\le C\cdot C_{\text{Rm}}\rho(x,y)^2,\\
        \Vert\partial_\alpha g(y)\Vert&\le C\cdot C_{\text{Rm}}\rho(x,y),\\
        \Vert\partial_{\alpha\beta}g(y)\Vert&\le C\cdot C_{\text{Rm}},
    \end{aligned}$$
    as long as $\rho(x,y)\le c/\sqrt{K}$. Here, $C>0$ is a universal constant, and $c>0$ only depends on $r_i$.
    \label{lem:local-g}
\end{lemma}

\textbf{Remark:} In the original paper \citep{xu2026polynomial}, the authors said that $c,C$ are polynomial in $d$. However, in their proof, the only property required was that $K\rho(x,y)^2$ is not too large, and the constants are actually universal.

In addition, under Assumption \ref{assumption:manifold2}, we can provide upper bounds for third and fourth derivatives.

\begin{lemma}
    Under Assumption \ref{assumption:manifold} and \ref{assumption:manifold2}, in the same setting of Lemma \ref{lem:local-g}, we further have
    $$\begin{aligned}
        \Vert\partial_{\alpha\beta\gamma}g(y)\Vert&\le C,\\
        \Vert\partial_{\alpha\beta\gamma\eta}g(y)\Vert&\le C,
    \end{aligned}$$
    where $C$ only depends on $C_{\text{Rm}}$ and $C_{\text{Rm}}'$.
    \label{lem:local-g2}
\end{lemma}

\begin{proof}
    In the original proof of Lemma \ref{lem:local-g} (Lemma 5 in \citealt{xu2026polynomial}), they define a matrix $\mJ$ along the geodesic curve $\gamma$ from $y$ to $x$ so that $g(x)=\mJ(1)^T\mJ(1)$ and the matrix Jacobi equation
    $$\mJ''(s)=-\mR(s)\mJ(s),\quad J(0)=0,\quad J'(0)=I.$$
    is satisfied. Here the matrix $\mR(s)$ is defined by $$\left(\mR(s)\vu\right)_\alpha:=\left\langle R\left(u^\beta e_\beta(s),\dot{\gamma}(s)\right)\dot{\gamma}(s),e_\alpha(s)\right\rangle,$$
    where $R$ represents the Riemmanian curvature tensor, and $\{e_\beta\}_{\beta=1}^n$ is an orthonormal frame obtained from the parallel transport along $\gamma$.

    In the proof of Lemma \ref{lem:local-g}, \citet{xu2026polynomial} have already shown that $\Vert \mJ\Vert\le C$ $\Vert\partial_\alpha\mJ\Vert\le C\rho(x,y)$, and $\Vert\partial_{\alpha\beta}\mJ\Vert\le C$, where the derivatives are taken with respect to the endpoint $x$. Now, differentiate the matrix Jacobi equation to the third order, we obtain that
    $$\partial_{\alpha\beta\gamma}\mJ''(s)=-\partial_{\alpha\beta\gamma}(\mR(s)\mJ(s)).$$
    Integrating twice with respect to $s$, and apply Assumption \ref{assumption:manifold2} to control the derivative of matrix $\mR$, we can obtain
    $$\partial_{\alpha\beta\gamma}\mJ(s)=\int_0^s(s-\tau)\left[C\partial_{\alpha\beta\gamma}\mJ(\tau)+C\cdot (C\rho(x,y)^2+C\rho(x,y)+C)\right]d\tau.$$
    Hence, by Gr\"onwall's inequality, we eventually obtain that $\Vert\partial_{\alpha\beta\gamma}\mJ(1)\Vert\le C$ whenever $\rho(x,y)\le\frac{c}{\sqrt{d}}$.
    Therefore, apply the product rule on $g(x)=\mJ(1)^T\mJ(1)$, $$\Vert\partial_{\alpha\beta\gamma}g(x)\Vert\le C\cdot C+C\cdot C\rho(x,y)\le C.$$
    The case for the fourth derivative is exactly the same, and the key is again that the constant in Assumption \ref{assumption:manifold2} does not depend on the dimension $d$.
\end{proof}

\end{document}